\documentclass[sn-mathphys-num]{sn-jnl}
\usepackage{subfig}
\usepackage{graphicx,version,color}
\usepackage{bm}
\usepackage{multirow}
\usepackage{graphicx}%
\usepackage{multirow}%
\usepackage{amsmath,amssymb,amsfonts}%
\usepackage{amsthm}%
\usepackage{mathrsfs}%
\usepackage[title]{appendix}%
\usepackage{xcolor}%
\usepackage{textcomp}%
\usepackage{manyfoot}%
\usepackage{booktabs}%
\usepackage[ruled,norelsize,vlined,linesnumbered]{algorithm2e}
\makeatletter
\newcommand{\removelatexerror}{\let\@latex@error\@gobble}
\makeatother
\usepackage{algpseudocode}%
\usepackage{listings}%
\usepackage{makecell}
\usepackage{natbib}
\DeclareMathOperator{\diag}{diag}

\newcommand{\Bone}{{\bm 1}}

\newcommand{\tF}{{\rm F}}

\newcommand{\pxi}{\partial_{x_i}}

\newcommand{\LP}{\mathcal{L}_{\rm{P}}}

\newcommand{\LS}{\mathcal{L}_{\rm{S}}}

\newcommand{\JP}{\mathcal{J}_{\rm{P}}}

\newcommand{\JS}{\mathcal{J}_{\rm{S}}}

\newcommand{\Eltwo}{\mathcal{E}_{\ell^2}}
\newcommand{\hEltwo}{\mathcal{E}'_{\ell^2}}

\theoremstyle{thmstyleone}%
\newtheorem{theorem}{Theorem}
\theoremstyle{thmstyletwo}%
\theoremstyle{thmstylethree}%

\begin{document}

\title[Deep Learning Optimization]{Deep Learning Optimization Using Self-Adaptive Weighted Auxiliary Variables}

\author[1]{\fnm{Yaru} \sur{Liu}}\email{yaruliu@std.uestc.edu.cn}

\author[1]{\fnm{Yiqi} \sur{Gu}}\email{yiqigu@uestc.edu.cn}

\author*[2]{\fnm{Michael K.} \sur{Ng}}\email{michael-ng@hkbu.edu.hk}

\affil[1]{\orgdiv{School of Mathematical Sciences}, \orgname{University of Electronic Science and Technology of China}, \orgaddress{\city{Sichuan}, \postcode{611731}, \country{China}}}

\affil[2]{\orgdiv{Department of Mathematics}, \orgname{Hong Kong Baptist University}, \orgaddress{\street{Kowloon Tong}, \city{Hong Kong}, \country{China}}}

\abstract{In this paper, we develop a new optimization framework for the least squares learning problem via fully connected neural networks or physics-informed neural networks. The gradient descent sometimes behaves inefficiently in deep learning because of the high non-convexity of loss functions and the vanishing gradient issue. Our idea is to introduce auxiliary variables to separate the layers of the deep neural networks and reformulate the loss functions for ease of optimization. We design the self-adaptive weights to preserve the consistency between the reformulated loss and the original mean squared loss, which guarantees that optimizing the new loss helps optimize the original problem. Numerical experiments are presented to verify the consistency and show the effectiveness and robustness of our models over gradient descent.}


\maketitle
\section{Introduction}\label{sec1}

In the past decade, deep learning has attracted significant interest from various mathematics and computer science research fields. For supervised learning, the least squares learning, i.e.,
\begin{equation}\label{00}
\min_\phi\frac{1}{N}\sum_{n=1}^N|\phi(x_n)-y_n|^2,
\end{equation}
is one of the most widely studied and applied models, where $\{(x_n,y_n)\}$ is a given dataset with $x_n$ representing the features and $y_n$ referring to the associated labels. The estimator $\phi$ is a deep neural network (DNN) to be trained. Since DNNs are highly nonlinear structures with their parameters, the mean squared error loss function in \eqref{00} might be highly non-convex with its variables and, therefore, difficult to solve.

Some recent work investigates the convergence behaviors of gradient descent optimizers in solving the optimization problem in \eqref{00}. Specifically, for $\phi$ being a fully-connected neural network (FNN), it was demonstrated that good minima can be found via gradient descent under over-parameterization hypothesis \cite{Du2019,Allen-Zhu2019,Zou2019,Du2019_2,Allen-Zhu2019_2,E2019,Zhou2021,Oymak2020}. Also, for the case that $\phi$ is a physics-informed neural network (PINN), which arises recently for problems involving partial differential equations (PDE) \cite{Raissi2019,Rao2020,Cai2021,Pang2019,Jagtap2021,Chiu2022}, researchers have
shown that gradient descent can find the global optima with over-parameterization \cite{Gao2023,Luo2020}.

However, over-parameterization involves unrealistically wide/deep neural networks that are usually impossible for practical implementation. And in practice, the high non-convexity of the loss function often impedes the gradient descent optimizer, which eventually converges to bad local optima. Moreover, training FNNs suffers from the vanishing gradient problem; namely, the gradient magnitudes are much smaller for the network parameters at the innermost layers, causing the Hessian of these parameters to be ill-conditioned and the gradient descent to be very inefficient, see \cite{Rognvaldsson1994,Erhan2009} for discussion.

To alleviate the preceding difficulties, Carreira-Perpinan et al. develop the model of auxiliary variables in \cite{Carreira-Perpinan2014}. Specifically, every layer of an FNN is represented by an intermediate variable so that the deep network architecture can be split into a sequence of shallow architectures. These new variables and equality constraints are added to the loss function by quadratic penalty terms, and the resulting optimization problem with auxiliary variables can be solved by standard optimizers. In \cite{Taylor2016,Wang2019}, an alternating direction method of multipliers (ADMM) is developed for this optimization, which avoids local minima by solving a sequence of sub-problems globally, leading to speedups in some cases. More applications of ADMM in deep learning can be referred to \cite{Gao2019,Yang2018,Sun2016,Song2023}. Overall, the method of auxiliary variables reduces the ill-conditioning caused by the deep network structure and partially decouples many variables, making efficient and distributed optimization possible.

However, despite having the same minimizers, the mean squared loss in \eqref{00} and the loss with auxiliary variables proposed in \cite{Carreira-Perpinan2014} are inconsistent; even if the latter becomes small, the former could remain large. Therefore, optimizing the loss with auxiliary variables may not decrease the mean squared error. Consequently, the model of auxiliary variables in \cite{Carreira-Perpinan2014} is less accurate or even completely ineffective for the original deep learning problem \eqref{00}. We observe in numerical experiments that the loss with auxiliary variables keeps falling using the alternating direction optimizer, but the learning error stays at a high level (see the numerical examples in Section \ref{sec_experiments}).

This work proposes a novel class of learning models using self-adaptive weighted auxiliary variables. Specifically, we develop self-adaptive weights to the quadratic penalty terms of the auxiliary variables, obtaining new weighted loss functions. We formulate these models for general FNNs and the PINNs associated with first-order linear PDEs. Moreover, we prove the consistency between the weighted loss and the original mean squared loss. The consistency constants are $O(L)$ for FNNs and $O(dL^2)$ for PINNs, where $L$ is the depth of the neural network and $d$ is the dimension of the problem. In addition, numerical experiments verify the proven consistencies. Compared with the conventional least squares model and the existing penalty model with auxiliary variables, the proposed model is more robust with random initialization and can obtain smaller learning errors. To our knowledge, these proposed models and their consistencies have not been studied in any previous work.

\subsection{Contributions}
The main contributions of this paper are listed as follows. 
\begin{itemize}
    \item We improve the FNN penalty model (PM-FNN) with auxiliary variables \cite{Carreira-Perpinan2014} by developing a self-adaptive weighted penalty model (SAPM-FNN). And we prove the consistency between SAPM-FNN and the conventional least squares FNN model (LS-FNN), which does not hold for PM-FNN.
    \item We propose the standard penalty model (PM-PINN) and the self-adaptive weighted penalty model (SAPM-PINN) for PINNs from first-order linear PDEs. We also prove the consistency between SAPM-PINN and the least squares PINN model (LS-PINN).
    \item The above consistencies are verified by numerical experiments, which also demonstrate that SAPM performs more robustly and accurately than PM and LS in both learning problems with FNNs and PDE problems with PINNs.
\end{itemize}
The models discussed in this paper are summarized in Table \ref{Tab_novelty}.
\begin{table}\small
\centering
\begin{tabular}{llc}
  \toprule
Models & Formulation & Comments \\\hline
LS-FNN & Eq. \eqref{LS_FNN}& conventional model \\
LS-PINN & Eq. \eqref{LS_PINN}& proposed in \cite{Raissi2019}\\
PM-FNN & Eq. \eqref{PM_FNN} & proposed in \cite{Carreira-Perpinan2014}  \\
PM-PINN & Eq. \eqref{PM_PINN} & proposed in this paper \\
SAPM-FNN & Eq. \eqref{SAPM_FNN} & proposed in this paper \\
SAPM-PINN & Eq. \eqref{SAPM_PINN} & proposed in this paper \\\bottomrule
\end{tabular}
\caption{\em Models discussed in this paper}\label{Tab_novelty} 
\end{table}

\subsection{Organization}
This paper is organized as follows. In Section \ref{sec_FNN}, we discuss the least squares learning problem with FNNs: we first review the penalty model with auxiliary variables and propose our self-adaptive weighted model. In Section \ref{sec_PINN}, we propose and develop the weighted model for the least squares problem with PINNs. Several numerical experiments are demonstrated in Section \ref{sec_experiments} to verify the theory and test the performance of these models. Conclusions and discussions about further research work are provided in Section \ref{sec5}.

\section{Fully-connected Neural Networks}\label{sec_FNN}
In this section, we consider the least squares problem \eqref{00} with $\phi(x_n;\theta)$ being an FNN. Such a problem appears widely in neural network-based regression.

\subsection{Least Squares Problems}
Suppose $d\geq1$ is the dimension and $\Omega$ is a subset of $\mathbb{R}^d$. Let $f:\Omega\rightarrow\mathbb{R}$ be an unknown target function. The learning problem is to identify $f$ given a dataset $\{(x_n,y_n)\}_{n=1}^N$, where $x_n\in\Omega$ is a feature vector and $y_n=f(x_n)$ is the corresponding label of $x_n$ for every $n$. By deep learning, one can identify $f$ by solving the optimization
\begin{equation}\label{01}
\min_\theta\frac{1}{N}\sum_{n=1}^N|\phi(x_n;\theta)-y_n|^2,
\end{equation}
where $\phi(x;\theta):\mathbb{R}^d\rightarrow\mathbb{R}$ is a deep neural network with input $x$ and parameter set $\theta$. 

Recall that an FNN $\phi$ of depth $L$ and width $M$ is defined as follows:
\begin{equation}\label{02}
\phi(x;\theta)=W_L\sigma(W_{L-1}\sigma(\dots\sigma(W_1x+b_1)\dots)+b_{L-1})+b_L,
\end{equation}
where $W_1\in \mathbb{R}^{M\times d}$, $W_2,\dots,W_{L-1}\in \mathbb{R}^{M\times M}$, $W_L\in \mathbb{R}^{1\times M}$ are weights; $b_1,\dots,b_{L-1}\in \mathbb{R}^{M\times 1}$, $b_L\in \mathbb{R}$ are biases; $\sigma$ is some activation function which is applied entry-wise to a vector to obtain another vector of the same size; $\theta=\{W_l,b_l\}_{l=1}^L$ is the set of all free parameters. Therefore, \eqref{01} can be rewritten as follows:
\begin{equation}\label{03}
\min_{W_l,b_l}\frac{1}{N}\sum_{n=1}^N\left|W_L\sigma(W_{L-1}\sigma(\dots\sigma(W_1x_n+b_1)\dots)+b_{L-1})+b_L-y_n\right|^2.
\end{equation}

For simplicity, we view $x_n$ as column vectors and use the notations
\begin{equation}
X:=[x_1~x_2~\dots~x_N]\in\mathbb{R}^{d\times N},\quad Y:=[y_1~y_2~\dots~y_N]\in\mathbb{R}^{1\times N}.
\end{equation}
Then \eqref{03} can rewritten again as follows:
\begin{equation}\label{LS_FNN}
\min_{W_l,b_l}\mathcal{L}:=\frac{1}{N}\|W_L\sigma(W_{L-1}\sigma(\dots\sigma(W_1X+b_1\Bone^\top)\dots)+b_{L-1}\Bone^\top)+b_L\Bone^\top-Y\|_\tF^2,
\end{equation}
where $\Bone$ is the all-one column vector of size $N$ and $\|\cdot\|_\tF$ is the Frobenius norm. We abbreviate the least squares deep learning model \eqref{LS_FNN} as LS-FNN.

\subsection{Models with Auxiliary Variables}\label{sec_PM_FNN}
Solving LS-FNN by common optimizers (e.g., gradient descent) is sometimes difficult. In particular, if the network is very deep, the mean squared loss $\mathcal{L}$ can be highly non-convex, and gradient descent may also suffer from vanishing gradient issues. In pioneering work \cite{Carreira-Perpinan2014}, the authors reformulated optimization \eqref{LS_FNN} by introducing auxiliary variables, each of which is for one layer of the neural network. Specifically, one considers the following constrained minimization,
\begin{equation}\label{05}
\begin{split}
\min_{W_l,b_l}&~\frac{1}{N}\|W_L\sigma(a_{L-1})+b_L\Bone^\top-Y\|_\tF^2\\
\text{s.t.}& \quad a_1 = W_1X+b_1\Bone^\top,\\
& \quad a_l=W_l\sigma(a_{l-1})+b_l\Bone^\top,\quad\forall~l=2,\dots,L-1,
\end{split}
\end{equation}
where $a_l\in\mathbb{R}^{M\times N}$ for $l=1,\dots,L-1$ are newly added auxiliary variables. We remark that the optimization \eqref{05} is not identical to that in \cite{Carreira-Perpinan2014}, where the bias vector is not involved, and $\sigma$ is located in the outermost part of the auxiliary variable. 

Note that the problem in \eqref{05} can be formulated as the following unconstrained minimization by adding $\ell^2$ penalty terms of the constraints to the loss function,
\begin{equation}\label{PM_FNN}
\begin{split}
&\min_{W_l,b_l,a_l} \LP:=\frac{1}{N}\Bigg[\|W_L\sigma(a_{L-1})+b_L\Bone^\top-Y\|_\tF^2\\
&+\sum_{l=2}^{L-1}\beta_l\|W_l\sigma(a_{l-1})+b_l\Bone^\top-a_l\|_\tF^2+\beta_1\|W_1X+b_1\Bone^\top-a_1\|_\tF^2\Bigg],
\end{split}
\end{equation}
where $\{\beta_l\}_{l=1}^{L-1}$ are prescribed positive constants that control the weights of these constraints. Compared with LS-FNN, the layer separation formulation \eqref{PM_FNN} decouples many variables, contributing to a less non-convex loss landscape. And it only has shallow architectures, hence alleviating the vanishing gradient problem to some extent. In particular, $\LP$ is convex in terms of the variables $\{W_l,b_l\}$. Even for $\{a_l\}$, they are applied with only one activation layer, which is much less non-convex than deep architectures. Consequently, it is more likely to find the global or ``good" local minimizers with a small training loss of the formulation \eqref{PM_FNN}. We abbreviate the penalty model \eqref{PM_FNN} as PM-FNN.

\subsection{Self-adaptive Weighted Auxiliary Variables}\label{sec_SAPM_FNN}
One limitation of PM-FNN is that the loss function $\LP$ is not consistent with the loss function $\mathcal{L}$ in LS-FNN. Even if a set of variables $\{W_l,b_l,a_l\}$ makes the squared loss function $\LP$ very close to 0, it cannot guarantee that $\mathcal{L}$ is also close to 0. To overcome this limitation, we propose the following formulation to improve PM-FNN:
\begin{equation}\label{SAPM_FNN}
\begin{split}
&\min_{W_l,b_l,a_l} \LS:=\frac{1}{N}\Bigg[\|W_L\sigma(a_{L-1})+b_L\Bone^\top-Y\|_\tF^2\\
&+\sum_{l=2}^{L-1}\beta_l\omega_l\|W_l\sigma(a_{l-1})+b_l\Bone^\top-a_l\|_\tF^2+\beta_1\omega_1\|W_1X+b_1\Bone^\top-a_1\|_\tF^2\Bigg].
\end{split}
\end{equation}
where 
\begin{equation}\label{27}
\omega_l=\prod_{j=l+1}^L\|W_j\|_\tF^2,\quad l=1,\ldots,L-1,
\end{equation}
are adaptive weights. The difference between $\LS$ and $\LP$ lies in these varying weights of the squared terms of the former, so we call \eqref{SAPM_FNN} self-adaptive weighted penalty model (SAPM-FNN) with auxiliary variables. We remark that the weights $\omega_l$ are explicitly represented by the variables $\{W_l\}$, so the variables $\{W_l\}$ are involved not only in the squared terms but in the weights. Fortunately, the loss function $\LS$ is still a quadratic form for each $W_l$, which allows closed-form minimizers in the alternating direction optimization (see Section \ref{sec_implementation_SAPM_FNN}).  

The following result guarantees that the mean squared loss $\mathcal{L}$ of LS-FNN is always bounded above by the weighted penalized loss $\LS$.

\begin{theorem}\label{thm01}
Suppose $\sigma$ is Lipschitz continuous, i.e., $|\sigma(z_1)-\sigma(z_2)|\leq B|z_1-z_2|$ for some $B>0$ and any $z_1,z_2\in\mathbb{R}$. Given $\mathcal{L}$ and $\LS$ defined in \eqref{LS_FNN} and \eqref{SAPM_FNN}, respectively. Then for all $\{W_l,b_l\}_{l=1}^L$ and $\{a_l\}_{l=1}^{L-1}$, it holds that
\begin{equation}\label{14}
\mathcal{L}\leq C_{B,\beta} \cdot L \cdot \LS,
\end{equation}
where $C_{B,\beta}=\max\{1,B^{2L-2}\}\cdot\max_{l=1,\dots,L-1}\{1,\beta_l^{-1}\}$.
\end{theorem}

\begin{proof}
By the triangle inequality, we have
\begin{multline}\label{09}
(N\mathcal{L})^{\frac{1}{2}}=\|W_L\sigma(W_{L-1}\sigma(\dots\sigma(W_1X+b_1\Bone^\top)\dots)+b_{L-1}\Bone^\top)+b_L\Bone^\top-Y\|_\tF\\
\leq \|W_L\sigma(a_{L-1})+b_L\Bone^\top-Y\|_\tF\\
+\|W_L\sigma(W_{L-1}\sigma(\dots\sigma(W_1X+b_1\Bone^\top)\dots)+b_{L-1}\Bone^\top)-W_L\sigma(a_{L-1})\|_\tF\\
\leq \|W_L\sigma(a_{L-1})+b_L\Bone^\top-Y\|_\tF\\
+\|W_L\|_\tF\cdot \|\sigma(W_{L-1}\sigma(\dots\sigma(W_1X+b_1\Bone^\top)\dots)+b_{L-1}\Bone^\top)-\sigma(a_{L-1})\|_\tF.
\end{multline}
By the Lipschitz continuity, it follows that
\begin{multline}\label{10}
\|\sigma(W_{L-1}\sigma(\dots\sigma(W_1X+b_1\Bone^\top)\dots)+b_{L-1}\Bone^\top)-\sigma(a_{L-1})\|_\tF\\
\leq B\|W_{L-1}\sigma(\dots\sigma(W_1X+b_1\Bone^\top)\dots)+b_{L-1}\Bone^\top-a_{L-1}\|_\tF.
\end{multline}
Therefore, we have
\begin{multline}\label{11}
(N\mathcal{L})^{\frac{1}{2}}\leq \|W_L\sigma(a_{L-1})+b_L\Bone^\top-Y\|_\tF\\
+B\|W_L\|_\tF\cdot \|W_{L-1}\sigma(\dots\sigma(W_1X+b_1\Bone^\top)\dots)+b_{L-1}\Bone^\top-a_{L-1}\|_\tF.
\end{multline}

Next, we handle the term $\|W_{L-1}\sigma(\dots\sigma(W_1X+b_1\Bone^\top)\dots)+b_{L-1}\Bone^\top-a_{L-1}\|_\tF$ by similar argument as in \eqref{09}-\eqref{10}. Then we obtain
\begin{multline}\label{12}
\|W_{L-1}\sigma(\dots\sigma(W_1X+b_1\Bone^\top)\dots)+b_{L-1}\Bone^\top-a_{L-1}\|_\tF\\
\leq \|W_{L-1}\sigma(a_{L-2})+b_{L-1}\Bone^\top-a_{L-1}\|_\tF\\
+B\|W_{L-1}\|_\tF\cdot \|W_{L-2}\sigma(\dots\sigma(W_1X+b_1\Bone^\top)\dots)+b_{L-2}\Bone^\top-a_{L-2}\|_\tF.
\end{multline}
Combining \eqref{11} and \eqref{12} leads to
\begin{multline}\label{13}
(N\mathcal{L})^{\frac{1}{2}}\leq \|W_L\sigma(a_{L-1})+b_L\Bone^\top-Y\|_\tF
+B\|W_L\|_\tF\cdot \|W_{L-1}\sigma(a_{L-2})+b_{L-1}\Bone^\top-a_{L-1}\|_\tF\\
+B^2\|W_L\|_\tF\|W_{L-1}\|_\tF\cdot \|W_{L-2}\sigma(\dots\sigma(W_1X+b_1\Bone^\top)\dots)+b_{L-2}\Bone^\top-a_{L-2}\|_\tF.
\end{multline}
Doing the preceding steps recursively, we finally have
\begin{equation}
\begin{split}
&(N\mathcal{L})^{\frac{1}{2}}\leq \|W_L\sigma(a_{L-1})+b_L\Bone^\top-Y\|_\tF\\
&+\sum_{l=2}^{L-1}B^{L-l}\left(\prod_{j=l+1}^L\|W_j\|_\tF\right)\|W_l\sigma(a_{l-1})+b_l\Bone^\top-a_l\|_\tF\\
&+B^{L-1}\left(\prod_{j=2}^L\|W_j\|_\tF\right)\|W_1X+b_1\Bone^\top-a_1\|_\tF,
\end{split}
\end{equation}
which leads to the desired result given in \eqref{14}.
\end{proof}

Note that for activations with the Lipschitz constant $B\leq1$ (e.g., ReLU activation), the constant $C_{B,\beta}$ is independent of $L$. Then Theorem \ref{thm01} implies that $\mathcal{L}\leq O(L)\LS$; namely, the conventional mean squared loss is no more than $O(L)$ times the weighted loss.

On the other hand, $\LS$ is also bounded by $\mathcal{L}$ in the following sense.
\begin{theorem}\label{thm02}
For all $\{W_l,b_l\}_{l=1}^L$, there exists $\{a_l\}_{l=1}^{L-1}$ such that
\begin{equation}
\LS\leq\mathcal{L}.
\end{equation}
\end{theorem}

\begin{proof}
It suffices to let $a_1 = W_1X+b_1\Bone^\top$ and $a_l=W_l\sigma(a_{l-1})+b_l\Bone^\top$ for $l=2,\dots,L-1$. In this case, the equality holds.
\end{proof}

Same as the penalized loss $\LP$, the weighted penalized loss $\LS$ remains convex in terms of every $W_l$ and $b_l$, and preserves shallow architectures of $\{a_l\}$. This allows us to develop efficient alternating direction optimizers (see Section \ref{sec_implementation_SAPM_FNN}). Comparatively, the mean squared loss $\mathcal{L}$ has deep architectures of the variables, bringing difficulties in finding ``good" optima using common optimizers. In Section \ref{sec_experiments}, we will show that SAPM-FNN in \eqref{SAPM_FNN} can achieve higher numerical accuracy in learning problems than PM-FNN and LS-FNN.

\subsection{Implementation}\label{sec_implementation_SAPM_FNN}
In the previous work \cite{Taylor2016,Wang2019}, alternating direction approaches are developed to minimize $\LP$, where minimization sub-problems for $\{W_l,b_l\}$ can be solved globally in closed form. We follow this strategy to develop algorithms for the minimization of $\LS$. For notational simplicity, we assume that $\beta_1=\dots=\beta_{L-1}=1$ in $\LS$, and the following discussion can be easily applied to general $\{\beta_l\}$.

In every iteration, we update $W_l$ with other variables fixed. To reduce $\LS$, it is equivalent to reduce 
\begin{equation}\label{17}
\|W_lA_l-P_l\|_\tF^2+\lambda_l\|W_l\|_\tF^2,
\end{equation}
where 
\begin{equation}
A_l=\begin{cases}\sigma(a_{l-1}),& 2\leq l\leq L,\\X,& l=1,\end{cases}\quad P_l=\begin{cases}Y-b_L\Bone^\top,& l=L,\\a_l-b_l\Bone^\top,& 1\leq l\leq L-1,\end{cases}
\end{equation}
and
\begin{equation}
\lambda_l=\begin{cases}\sum_{i=1}^{l-1}\left(\prod_{j=i+1}^{l-1}\|W_j\|_\tF^2\right)\|W_iA_i-P_i\|_\tF^2, & 2\leq l\leq L,\\ 0,& l = 1,\end{cases}
\end{equation}
are all fixed. Note that the minimizer $W_l$ of \eqref{17} is exactly the least square solution of the linear system 
\begin{equation}\label{25}
W_l\begin{bmatrix}A_l & \sqrt{\lambda_l}I\end{bmatrix}=\begin{bmatrix}P_l & O\end{bmatrix},
\end{equation}
where $I$ is the identity matrix and $O$ is the zero matrix. Therefore, in every iteration, we can update $W_l$ by solving \eqref{25} in the least square sense.

Similarly, $b_l$ is updated as the least square solution of the linear system 
\begin{equation}\label{26}
b_l\Bone^\top=\begin{cases}Y-W_lA_l,& l=L,\\a_l-W_lA_l,& 1\leq l\leq L-1.\end{cases}
\end{equation}
Note that its least square solution is exactly the mean of the columns on the right hand side.

However, due to the nonlinearity of the activation function $\sigma$, the minimization sub-problems for $\{a_l\}$ can not be solved in closed form in general (although closed forms exist for some simple activations such as ReLU activation \cite{Taylor2016}). So we propose to update $\{a_l\}$ by gradient descent using a line search subroutine. Note that each $a_l$ is involved only in two terms of $\LS$, and reducing $\LS$ is equivalent to reducing
\begin{equation}
\mathcal{S}_l(a_l):=\|W_{l+1}\sigma(a_l)-P_{l+1}\|_\tF^2+\|W_{l+1}\|_\tF^2\|W_lA_l+b_l\Bone^\top-a_l\|_\tF^2,
\end{equation}
whose gradient is
\begin{equation}
\begin{split}
\nabla_{a_l}\mathcal{S}_l=2\Big[\left(W_{l+1}^\top(W_{l+1}\sigma(a_l)-P_{l+1})\right)*\sigma'(a_l)+\|W_{l+1}\|_\tF^2(a_l-W_lA_l-b_l\Bone^\top)\Big],
\end{split}
\end{equation}
given that $\sigma$ is differentiable. Here $*$ means entry-wise multiplication of two matrices of the same size.

In summary, we propose the algorithm that alternatively updates $W_l$, $b_l$, $a_l$ for every $l$ from large to small (see Algorithm \ref{alg01}). We use $\tau>0$ to denote the learning rate in the gradient descent update. Note that $\tau$ can be various and adaptive for different updates. 

\begin{algorithm}
\DontPrintSemicolon
\KwIn{data $X$, $Y$; number of iterations $N_k$; learning rate $\tau$}
\KwOut{a feasible solution $\{W_l,b_l,a_l\}$.}
\Begin{Initialize $\{W_l,b_l,a_l\}_{l=1}^L$\;
\For{$k = 0,\cdots,N_k-1$}{
Solve $W_L\begin{bmatrix}A_L & \sqrt{\lambda_L}I\end{bmatrix}=\begin{bmatrix}P_L & O\end{bmatrix}$ for $W_L$\;
Solve $b_L\Bone^\top=Y-W_LA_L$ for $b_L$\;
\For{$l = L-1,\cdots,1$}{
$a_l\leftarrow a_l-\tau\nabla_{a_l}\mathcal{S}_l$,\;
Solve $W_l\begin{bmatrix}A_l & \sqrt{\lambda_l}I\end{bmatrix}=\begin{bmatrix}P_l & O\end{bmatrix}$ for $W_l$\;
Solve $b_l\Bone^\top=a_l-W_lA_l$ for $b_l$\;}}
return $\{W_l,b_l,a_l\}_{l=1}^L$}
\caption{Solve $\min_{W_l,b_l,a_l} \LS$\label{alg01}}
\end{algorithm}

In the previous work \cite{Taylor2016,Wang2019}, the linear system for $\{W_l\}$ is formulated as $W_lA_l=P_l$, which may suffer from ill-conditioning. Recall that $A_l=\sigma(a_{l-1})$ for $l\geq2$. Suppose $\sigma$ is the ReLU activation. If $a_{l-1}$ has rows whose entries are all negative, then the corresponding rows of $\sigma(a_{l-1})$ are zero, making row rank deficient. The similar issue appears for sigmoidal activations that $\sigma(a_{l-1})$ has nearly zero rows if $a_{l-1}$ has negative rows with large moduli. However, this issue is almost overcome in our formulation. Thanks to the existence of $\sqrt{\lambda_n}I$ in the coefficient matrix in \eqref{25}, the linear system has row full rank for $l\geq2$ when $\lambda_n$ is hardly to be zero.

\section{Physics-informed Neural Networks}\label{sec_PINN}
In this section, we consider the least squares problem in \eqref{00} with $\phi(x_n;\theta)$ being a PINN. PINNs are the resulting functions when applying physically relevant differential operators to standard neural networks (e.g., FNNs). For convenience, we only consider the PINNs derived from first-order linear equations in this paper. Models for other types of PDEs can be formulated in similar ways. 

\subsection{Least Squares Models for First-Order Linear Equations}
Suppose $d\geq1$ is the dimension and $\Omega$ is a smooth domain in $\mathbb{R}^d$. We consider the following first-order linear equation:
\begin{gather}
c(x)\cdot\nabla u(x)=f(x),\quad\text{for}~x\in\Omega,\label{15}\\
u(x)=g(x),\quad\text{for}~x\in\Gamma,\label{16}    
\end{gather}
where $u:\Omega\rightarrow\mathbb{R}$ is the unknown solution; $c:\Omega\rightarrow\mathbb{R}^d$ is a given vector-valued function with each component being $c_i$; $f:\Omega\rightarrow\mathbb{R}$ is a given data function; $\Gamma\subset\partial\Omega$ is a subset of the boundary; $g:\Gamma\rightarrow\mathbb{R}$ is a given initial/boundary value function. For physical steady-state problems, $x$ is a $d$-dimensional spatial variable, and \eqref{16} specifies the Dirichlet condition on part of the boundary. For physical time-dependent problems, $x$ consists of a one-dimensional time variable and a $(d-1)$-dimensional spatial variable, and \eqref{16} formulates the initial and Dirichlet boundary conditions. 

To solve the problem via deep learning, one can use a neural network to approximate the solution. We introduce an FNN $\psi(x;\theta):\mathbb{R}^d\rightarrow\mathbb{R}$ to approximate $u$ in $\Omega$. Recall $\theta=\{W_l,b_l\}_{l=1}^L$ is the set of neural network parameters shown in \eqref{02}. Then applying the differential operator in \eqref{15} to $\psi$ leads to a PINN $\phi$, namely,
\begin{equation}\label{21}
\phi(x;\theta):=c(x)\cdot\nabla \psi(x;\theta).
\end{equation}
Next, $\phi$ (or $\psi$) can be trained by the least squares regression using the data derived from the problem \eqref{15}-\eqref{16}. For the PDE \eqref{15}, we let $\{x_n^{(1)}\}_{n=1}^{N_1}$ be a set of feature points in $\Omega$, and let $y_n^{(1)}=f(x_n^{(1)})$ for every $n$, then $\{(x_n^{(1)},y_n^{(1)})\}_{n=1}^{N_1}$ forms a training dataset. The least squares model for $\phi$ is given by
\begin{equation}\label{18}
\min_\theta\frac{1}{N_1}\sum_{n=1}^{N_1}\left|\phi(x_n^{(1)};\theta)-y_n^{(1)}\right|^2.
\end{equation}
Similarly, for the initial/boundary condition \eqref{16}, we let $\{x_n^{(2)}\}_{n=1}^{N_2}$ be a set of sample points in $\Gamma$, and let $y_n^{(2)}$ be the label of $x_n^{(2)}$ defined by $y_n^{(2)} = g(x_n^{(2)})$,
then $\{(x_n^{(2)},y_n^{(2)})\}_{n=1}^{N_2}$ forms another training dataset. The least squares model for $\psi$ is conducted by
\begin{equation}\label{19}
\min_\theta\frac{1}{N_2}\sum_{n=1}^{N_2}\left|\psi(x_n^{(2)};\theta)-y_n^{(2)}\right|^2.
\end{equation}

Combining \eqref{18} and \eqref{19} leads to the following least squares PINN model (LS-PINN) for the problem in \eqref{15}-\eqref{16}:
\begin{multline}\label{28}
\min_\theta\mathcal{J}:=\frac{1}{N_1}\sum_{n=1}^{N_1}\left|c(x_n^{(1)})\cdot\nabla \psi(x_n^{(1)};\theta)-y_n^{(1)}\right|^2+\frac{\mu}{N_2}\sum_{n=1}^{N_2}|\psi(x_n^{(2)};\theta)-y_n^{(2)}|^2,
\end{multline}
where $\mu>0$ is a scalar weight.

We follow the FNN representation \eqref{02} for $\phi$, and use the notations
\begin{equation}
\begin{split}
X^{(k)}:=&[x_1^{(k)}~x_2^{(k)}~\dots~x_{N_k}^{(k)}]\in\mathbb{R}^{d\times N_k},\\
Y^{(k)}:=&[y_1^{(k)}~y_2^{(k)}~\dots~y_{N_k}^{(k)}]\in\mathbb{R}^{1\times N_k}
\end{split}
\end{equation}
for $k=1$ and $2$.
Then \eqref{28} can be rewritten as follows:
\begin{equation}\label{LS_PINN}
\min_{W_l,b_l}\mathcal{J}=\mathcal{J}^{(1)}+\mu\mathcal{J}^{(2)},
\end{equation}
with
\begin{multline*}
\mathcal{J}^{(1)}:=\frac{1}{N_1}\Big\|\sum_{i=1}^d c_i(X^{(1)})*\pxi\Big(W_L\sigma(W_{L-1}\sigma(\dots\sigma(W_1X^{(1)}+b_1\Bone^\top)\dots)\\
+b_{L-1}\Bone^\top)+b_L\Bone^\top\Big)-Y^{(1)}\Big\|_\tF^2,
\end{multline*}
and
\begin{equation*}
    \mathcal{J}^{(2)}:=\frac{1}{N_2}\|W_L\sigma(W_{L-1}\sigma(\dots\sigma(W_1X^{(2)}+b_1\Bone^\top)\dots)+b_{L-1}\Bone^\top)+b_L\Bone^\top-Y^{(2)}\|_\tF^2.
\end{equation*}

\subsection{Models with Auxiliary Variables}
Note that PINNs are the differentiation of FNNs under specific differential operators. If the FNNs are deep, the corresponding PINNs are also deep nonlinear structures with their parameters. Consequently, the strategy of PM-FNN discussed in Section \ref{sec_PM_FNN} can also be used to improve LS-PINN. 

We introduce the auxiliary variables 
\begin{equation}\label{22}
a_l^{(k)}=\begin{cases}
W_1X^{(k)}+b_1\Bone^\top,\quad l=1,\\
W_l\sigma(a_{l-1}^{(k)})+b_l\Bone^\top,\quad l=2,\dots,L,
\end{cases}
\end{equation}
for $k=1$ and $2$.

We consider the representation of $\pxi \phi$ in terms of auxiliary variables, where  $\pxi $ means taking the derivative of the $i$-th component of $x$. For this purpose, we introduce another set of auxiliary variables, i.e.
\begin{equation}\label{23}
d_{l,i}=\pxi a_l^{(1)},\quad l=1,\cdots,L.
\end{equation}
Using the expression \eqref{22} and chain rule, it follows \eqref{23} that 
\begin{equation}\label{24}
d_{l,i}=\begin{cases}
W_1(:,i)\Bone^\top,\quad l=1,\\
W_l\left(\sigma'(a_{l-1}^{(1)})*d_{l-1,i}\right),\quad l=2,\cdots,L,
\end{cases}
\end{equation}
where $W_1(:,i)$ means the $i$-th column of $W_1$. Note that 
\begin{equation}
a_L^{(k)}=\left[\phi(x_1^{(1)})~\phi(x_2^{(1)})~\cdots~\phi(x_{N_1}^{(1)})\right]
\end{equation}
and
\begin{equation}
d_{L,i}=\left[\pxi \phi(x_1^{(1)})~\pxi \phi(x_2^{(1)})~\cdots~\pxi \phi(x_{N_1}^{(1)})\right],
\end{equation}
the minimization \eqref{LS_PINN} can be rewritten as the constrained form
\begin{equation}\label{20}
\begin{split}
\min_{W_l,b_l} &~\frac{1}{N_1}\left\|\sum_{i=1}^{d}c_{i}(X^{(1)})*d_{L,i}-Y^{(1)}\right\|_\tF^2+\frac{\mu}{N_2}\left\|a_L^{(2)}-Y^{(2)}\right\|_\tF^2\\
\text{s.t.} &~\eqref{22}~\text{and}~\eqref{24},
\end{split}
\end{equation}
where $c_i(X^{(1)}):=\left[c_i(x_1^{(1)})~c_i(x_2^{(1)})~\cdots~c_i(x_{N_1}^{(1)})\right]$.

By similar argument in Section \ref{sec_FNN}, we can derive the standard penalty model (PM-PINN) from \eqref{20}, which is
\begin{equation}\label{PM_PINN}
\min_{W_l,b_l,a_l^{(1)},a_l^{(2)},d_l} \JP:=\JP^{(1)}+\mu \JP^{(2)}
\end{equation}
with
\begin{multline}
\JP^{(1)}:=\frac{1}{N_1}\Bigg[\left\|\sum_{i=1}^{d}c_{i}(X^{(1)})*d_{L,i}-Y^{(1)}\right\|_\tF^2\\
+\sum_{l=2}^{L-1}\beta_l^{(1)}\left\|W_l\sigma(a_{l-1}^{(1)})+b_l\Bone^\top-a_l^{(1)} \right\|_\tF^2+\beta_1^{(1)}\left\|W_1X^{(1)}+b_1\Bone^\top-a_1^{(1)}\right\|_\tF^2\\
+\sum_{l=2}^L\alpha_l^{(1)}\sum_{i=1}^{d}\left\|W_l\left(\sigma'(a_{l-1}^{(1)})*d_{l-1,i}\right)-d_{l,i}\right\|_\tF^2
+\alpha_1^{(1)}\sum_{i=1}^{d}\left\|W_1(:,i)\Bone^\top-d_{1,i}\right\|_\tF^2\Bigg]
\end{multline}
and
\begin{multline}
\JP^{(2)}:=\frac{1}{N_2}\Bigg[\left\|W_L\sigma(a_{L-1}^{(2)})+b_L\Bone^\top-Y^{(2)}\right\|_\tF^2\\
+\sum_{l=2}^{L-1}\beta_l^{(2)}\left\|W_l\sigma(a_{l-1}^{(2)})+b_l\Bone^\top-a_l^{(2)} \right\|_\tF^2+\beta_1^{(2)}\left\|W_1X^{(2)}+b_1\Bone^\top-a_1^{(2)}\right\|_\tF^2\Bigg],
\end{multline}
where $\mu$, $\beta_l^{(1)}$, $\alpha_l^{(1)}$, $\beta_l^{(2)}$ are fixed positive scalar weights. 

Like the gap between PM-FNN and LS-FNN, the models PM-PINN and LS-PINN are inconsistent. So we can not optimize $\mathcal{J}$ by actually optimizing $\JP$. Numerical results also show that the PDE solution error may not reduce even if $\JP$ is decreased to a low level (see Section \ref{sec_experiments}).

\subsection{Self-adaptive Weighted Auxiliary Variables}\label{sec_SA_PINN}
Following the strategy of SAPM-FNN discussed in Section \ref{sec_SAPM_FNN}, we introduce self-adaptive weights to improve PM-PINN. Specifically, we define diagonal weight matrices 
\begin{equation*}
\Omega_l:=\diag\left(\sqrt{\sum_{i=1}^d\sum_{j=l}^{L-1}\|d_{j,i}[1]\|_\tF^2} ,\dots,\sqrt{\sum_{i=1}^d\sum_{j=l}^{L-1}\|d_{j,i}[N_1]\|_\tF^2}\right),
\end{equation*}
for $l=1,\dots,L-1$, where the notation $A[n]$ means the $n$-th column of a matrix $A$. Also, we extend the definition \eqref{27} of $w_l$ by letting $\omega_L=1$. Then we propose the following self-adaptive weighted penalty model (SAPM-PINN):
\begin{equation}\label{SAPM_PINN}
\min_{W_l,b_l,a_l^{(1)},a_l^{(2)},d_{l,i}} \JS:=\JS^{(1)}+\mu \JS^{(2)}
\end{equation}
with
\begin{multline}
\JS^{(1)}:=\frac{1}{N_1}\Bigg[\Big\|\sum_{i=1}^d c_i(X^{(1)})*d_{L,i}-Y^{(1)}\Big\|^2\\
+\sum_{l=2}^{L-1}\beta_l^{(1)}\omega_l\Big\|\Big(W_l\sigma(a_{l-1}^{(1)})+b_l\Bone^\top-a_l^{(1)}\Big)\Omega_l\Big\|_\tF^2+\beta_1^{(1)}\omega_1\Big\|\Big(W_1X^{(1)}+b_1\Bone^\top-a_1^{(1)}\Big)\Omega_1\Big\|_\tF^2\\
+\sum_{l=2}^{L}\alpha_l^{(1)}\omega_l\sum_{i=1}^d\|W_l(\sigma^{'}(a_{l-1}^{(1)})*d_{l-1,i})-d_{l,i}\|_\tF^2+\alpha_1^{(1)}\omega_1\sum_{i=1}^d\|W_1(:,i)\Bone^\top-d_{1,i}\|_\tF^2\Bigg],
\end{multline}
and
\begin{multline}
\JS^{(2)}:=\frac{1}{N_2}\Bigg[\|W_L\sigma(a_{L-1}^{(2)})+b_L\Bone^\top-Y^{(2)}\|_\tF^2\\
+\sum_{l=2}^{L-1}\beta_l^{(2)}\omega_l\|W_l\sigma(a_{l-1}^{(2)})+b_l\Bone^\top-a_l^{(2)}\|_\tF^2+\beta_1^{(2)}\omega_1\|W_1X^{(2)}+b_1\Bone^\top-a_1^{(2)}\|_\tF^2\Bigg].
\end{multline}

The consistency between $\JS$ and $\mathcal{J}$ is illustrated by the following theory.

\begin{theorem}\label{thm03}
Suppose $\sigma$ and $\sigma^{'}$ are Lipschitz continuous, and $\sigma^{'}$ is a bounded function, i.e., $|\sigma(z_1)-\sigma(z_2)|\leq C_{\sigma}|z_1-z_2|$, $|\sigma^{'}(z_1)-\sigma^{'}(z_2)|\leq C_{\sigma^{'}}|z_1-z_2|$, and $|\sigma^{'}(z_1)|\leq B_{\sigma^{'}}$, for some $C_{\sigma}, C_{\sigma^{'}}, B_{\sigma^{'}}>0$, any $z_1,z_2\in\mathbb{R}$.  Also, suppose $c_i$ is a bounded function, i.e., $|c_i(z)|\leq B_c$ for some $B_c>0$, any $z\in\mathbb{R}^d$, $i=1,\ldots,d$. Given $\mathcal{J}$ and $\JS$ defined in \eqref{LS_PINN} and \eqref{SAPM_PINN}, respectively. Then for all $\{W_l,b_l\}_{l=1}^L$, $\{a_l^{(1)},a_l^{(2)}\}_{l=1}^{L-1}$ and $\{d_{l,i}\}_{l=1,\dots,L}^{i=1,\ldots,d}$, it holds that
\begin{equation}\label{eq:00}
\mathcal{J}\leq \widetilde{C}  d  L (L+1) \JS,
\end{equation}
where 
\begin{multline}
\widetilde{C}=\max\Bigg\{\max_{l=1,\dots,L}\Big\{1,(\alpha_l^{(1)})^{-1},(\beta_l^{(1)})^{-1}\Big\}
\\
\cdot\max\Big\{1,B_c^2,B_c^2 B_{\sigma^{'}}^{2L-2},B_c^2 C_{\sigma^{'}}^2\max\{B_{\sigma^{'}}^{2L-4},C_{\sigma}^{2L-4}\},B_c^2 C_{\sigma^{'}}^2\max\{B_{\sigma^{'}}^{2L-6},C_{\sigma}^{2L-6}\}\Big\},\\
\max_{l=1,\dots,L}\Big\{1,(\beta_l^{(2)})^{-1}\Big\}\cdot \max\{1,C_\sigma^{2L-2}\}\Bigg\}.
\end{multline}
\end{theorem}
\begin{proof}
Let 
\begin{equation}
\mathcal{J}_n^{(1)}=\Big|\sum_{i=1}^d c_i(x_n^{(1)})\pxi(W_L\sigma(W_{L-1}\sigma(\dots\sigma(W_1x_n^{(1)}+b_1)\dots)+b_{L-1})+b_L)-y_n^{(1)}\Big|^2.
\end{equation}
Then $\mathcal{J}^{(1)}=\frac{1}{N_1}\sum_{n=1}^{N_1}\mathcal{J}_n^{(1)}$. Using the triangle inequality, we have
\begin{multline}\label{eq:01}
    \Big(\mathcal{J}_n^{(1)}\Big)^{\frac 12}=\Big|\sum_{i=1}^d c_i(x_n^{(1)})\pxi(W_L\sigma(W_{L-1}\sigma(\dots\sigma(W_1x_n^{(1)}+b_1)\dots)+b_{L-1})+b_L)\\
    -y_n^{(1)}\Big|\leq\sum_{i=1}^d |c_i(x_n^{(1)})| |\mathcal{S}_{n,L}|+\Big|\sum_{i=1}^d c_i(x_n^{(1)})d_{L,i}[n]-y_n^{(1)}\Big|
\end{multline}
where $\mathcal{S}_{n,L}:=\pxi(W_L\sigma(W_{L-1}\sigma(\dots\sigma(W_1x_n^{(1)}+b_1)\dots)
+b_{L-1})+b_L)-d_{L,i}[n]$.

Using the chain rule and triangle inequality again, we obtain 
\begin{multline}\label{eq:02}
|\mathcal{S}_{n,L}|=\Big|W_L\Big(\sigma^{'}(W_{L-1}\sigma(\dots\sigma(W_1x_n^{(1)}+b_1)\dots)
+b_{L-1})\\ *\pxi(W_{L-1}\sigma(\dots
\sigma(W_1x_n^{(1)}+b_1)\dots)+b_{L-1})\Big)-d_{L,i}[n]\Big|\\
\leq\Big|W_L\Big(\sigma^{'}(W_{L-1}\sigma(\dots\sigma(W_1x_n^{(1)}+b_1)\dots)
+b_{L-1})*\pxi(W_{L-1}\sigma(\dots\\
\sigma(W_1x_n^{(1)}+b_1)\dots)+b_{L-1})\Big)-W_L\Big(\sigma^{'}(a_{L-1}^{(1)}[n])*d_{L-1,i}[n]\Big)\Big|\\
+\Big|W_L\Big(\sigma^{'}(a_{L-1}^{(1)}[n])*d_{L-1,i}[n]\Big)-d_{L,i}[n]\Big|\\
\leq \|W_L\|_\tF\|\mathcal{T}_{L-1,i}\|_\tF+\left|W_L(\sigma^{'}(a_{L-1}^{(1)}[n])*d_{L-1,i}[n])-d_{L,i}[n]\right|,
\end{multline}
where 
\begin{multline}
\mathcal{T}_{L-1,i}:=\sigma^{'}(W_{L-1}\sigma(\dots\sigma(W_1x_n^{(1)}+b_1)\dots)
+b_{L-1})\\
*\pxi\left(W_{L-1}\sigma(\dots
\sigma(W_1x_n^{(1)}+b_1)\dots)+b_{L-1}\right)-\sigma^{'}(a_{L-1}^{(1)}[n])*d_{L-1,i}[n].
\end{multline}
According to the Lipschitz continuity and the triangle inequality, it follows that
\begin{multline}
\|\mathcal{T}_{L-1,i}\|_\tF\leq\|\sigma^{'}(W_{L-1}\sigma(\dots\sigma(W_1x_n^{(1)}+b_1)\dots)
+b_{L-1})*\pxi(W_{L-1}\sigma(\dots\\
\sigma(W_1x_n^{(1)}+b_1)\dots)+b_{L-1}))-\sigma^{'}(W_{L-1}\sigma(\dots\sigma(W_1x_n^{(1)}+b_1)\dots)
+b_{L-1})*d_{L-1,i}[n]\|_\tF\\
+\|\sigma^{'}(W_{L-1}\sigma(\dots\sigma(W_1x_n^{(1)}+b_1)\dots)
+b_{L-1})*d_{L-1,i}[n]-\sigma^{'}(a_{L-1}^{(1)}[n])*d_{L-1,i}[n]\|_\tF\\
\leq\|\sigma^{'}(W_{L-1}\sigma(\dots\sigma(W_1x_n^{(1)}+b_1)\dots)
+b_{L-1})\|_\infty\|\mathcal{S}_{n,L-1}\|_\tF\\
+\|\sigma^{'}(W_{L-1}\sigma(\dots\sigma(W_1x_n^{(1)}+b_1)\dots)
+b_{L-1})-\sigma^{'}(a_{L-1}^{(1)}[n])\|_\tF\|d_{L-1,i}[n]\|_\tF\\
\leq B_{\sigma^{'}}\|\mathcal{S}_{n,L-1}\|_\tF+C_{\sigma^{'}}\|\mathcal{R}_{L-1}\|_\tF\|d_{L-1,i}[n]\|_\tF,
\end{multline}
where $\mathcal{R}_{L-1}:=W_{L-1}\sigma(\dots\sigma(W_1x_n^{(1)}+b_1)\dots)
+b_{L-1}-a_{L-1}^{(1)}[n]$. Therefore, we get
\begin{multline}\label{eq:03}
|\mathcal{S}_{n,L}|\leq B_{\sigma^{'}}\|W_L\|_\tF\|\mathcal{S}_{n,L-1}\|_\tF+C_{\sigma^{'}}\|W_L\|_\tF\|\mathcal{R}_{L-1}\|_\tF\|d_{L-1,i}[n]\|_\tF\\
+\left|W_L(\sigma^{'}(a_{L-1}^{(1)}[n])*d_{L-1,i}[n])-d_{L,i}[n]\right|.
\end{multline}
By similar argument as in \eqref{eq:02}-\eqref{eq:03}, we obtain
\begin{multline}\label{eq:04}
\|\mathcal{S}_{n,L-1}\|_\tF\leq B_{\sigma^{'}}\|W_{L-1}\|_\tF\|\mathcal{S}_{n,L-2}\|_\tF+C_{\sigma^{'}}\|W_{L-1}\|_\tF\|\mathcal{R}_{L-2}\|_\tF\|d_{L-2,i}[n]\|_\tF\\
+\|W_{L-1}(\sigma^{'}(a_{L-2}^{(1)}[n])*d_{L-2,i}[n])-d_{L-1,i}[n]\|_\tF.
\end{multline}
Combining \eqref{eq:03}-\eqref{eq:04} leads to
\begin{multline}\label{eq:05}
|\mathcal{S}_{n,L}|\leq B_{\sigma^{'}}^2\|W_L\|_\tF\|W_{L-1}\|_\tF\|\mathcal{S}_{n,L-2}\|_\tF\\
+ B_{\sigma^{'}}C_{\sigma^{'}}\|W_L\|_\tF\|W_{L-1}\|_\tF\|\mathcal{R}_{L-2}\|_\tF\|d_{L-2,i}[n]\|_\tF\\
+ B_{\sigma^{'}}\|W_L\|_\tF\|W_{L-1}(\sigma^{'}(a_{L-2}^{(1)}[n])*d_{L-2,i}[n])-d_{L-1,i}[n]\|_\tF\\
+C_{\sigma^{'}}\|W_L\|_\tF\|\mathcal{R}_{L-1}\|_\tF\|d_{L-1,i}[n]\|_\tF
+\left|W_L(\sigma^{'}(a_{L-1}^{(1)}[n])*d_{L-1,i}[n])-d_{L,i}[n]\right|.
\end{multline}
Doing the preceding steps recursively, we have
\begin{multline}\label{eq:06}
|\mathcal{S}_{n,L}|\leq \sum_{l=2}^{L-1} B_{\sigma^{'}}^{L-l}\Big(\prod_{k=l+1}^L\|W_k\|_\tF\Big)
\|W_l(\sigma^{'}(a_{l-1}^{(1)}[n])*d_{l-1,i}[n])-d_{l,i}[n]\|_\tF\\
+B_{\sigma^{'}}^{L-1}\Big(\prod_{k=2}^L\|W_k\|_\tF\Big)\|W_1(:,i)-d_{1,i}[n]\|_\tF
+\left|W_L(\sigma^{'}(a_{L-1}^{(1)}[n])*d_{L-1,i}[n])-d_{L,i}[n]\right|\\
+C_{\sigma^{'}}\sum_{l=1}^{L-1}B_{\sigma^{'}}^{L-l-1}\|d_{l,i}[n]\|_\tF\Big(\prod_{k=l+1}^L\|W_k\|_\tF\Big)\|\mathcal{R}_{l}\|_\tF.
\end{multline}
Next, using a similar argument as in the proof of Theorem \ref{thm01}, we have
\begin{multline}\label{eq:07}
    \|\mathcal{R}_{l}\|_\tF\leq \sum_{j=2}^{l-1}C_{\sigma}^{l-j}\Big(\prod_{k=j+1}^l\|W_k\|_\tF\Big)\|W_j\sigma(a^{(1)}_{j-1}[n])+b_j-a^{(1)}_j[n]\|_\tF\\
    +C_{\sigma}^{l-1}\Big(\prod_{k=2}^l\|W_k\|_\tF\Big)\|W_1x_n^{(1)}+b_1-a^{(1)}_1[n]\|_\tF+\|W_l\sigma(a_{l-1}^{(1)}[n])+b_l-a_l^{(1)}[n]\|_\tF,
\end{multline}
for $l=3,\ldots,L-1$, and
\begin{gather}\label{eq:08}
    \|\mathcal{R}_{2}\|_\tF\leq C_{\sigma}\|W_2\|_\tF\|W_1x_n^{(1)}+b_1-a^{(1)}_1[n]\|_\tF
    +\|W_2\sigma(a_1^{(1)}[n])+b_2-a_2^{(1)}[n]\|_\tF,
\end{gather}
and
\begin{equation}\label{eq:09}
    \begin{aligned}
    \|\mathcal{R}_{1}\|_\tF=\|W_1x_n^{(1)}+b_1-a^{(1)}_1[n]\|_\tF.
    \end{aligned}
\end{equation}
Combining \eqref{eq:06}-\eqref{eq:09} and bringing it into \eqref{eq:01} yields
\begin{multline}\label{eq:010}
\Big(\mathcal{J}_n^{(1)}\Big)^{\frac 12}\leq \sum_{i=1}^d B_c\Bigg[B_{\sigma^{'}}^{L-1}\Big(\prod_{k=2}^L\|W_k\|_\tF\Big)\|W_1(:,i)-d_{1,i}[n]\|_\tF\\
+\sum_{l=2}^{L-1}B_{\sigma^{'}}^{L-l}\Big(\prod_{k=l+1}^L\|W_k\|_\tF\Big)
\|W_l(\sigma^{'}(a_{l-1}^{(1)}[n])*d_{l-1,i}[n])-d_{l,i}[n]\|_\tF\\
+C_{\sigma^{'}}\sum_{l=2}^{L-1}B_{\sigma^{'}}^{L-l-1} \|d_{l,i}[n]\|_\tF \sum_{j=2}^{l}C_{\sigma}^{l-j}\Big(\prod_{k=j+1}^L\|W_k\|_\tF\Big)\|W_j\sigma(a^{(1)}_{j-1}[n])+b_j-a^{(1)}_j[n]\|_\tF\\
+C_{\sigma^{'}}\sum_{l=1}^{L-1}B_{\sigma^{'}}^{L-l-1} C_{\sigma}^{l-1}\|d_{l,i}[n]\|_\tF \Big(\prod_{k=2}^L\|W_k\|_\tF\Big)\|W_1x_n^{(1)}+b_1-a^{(1)}_1[n]\|_\tF\\
+\left|W_L(\sigma^{'}(a_{L-1}^{(1)}[n])*d_{L-1,i}[n])-d_{L,i}[n]\right|\Bigg]
+\Big|\sum_{i=1}^d c_i(x_n^{(1)})d_{L,i}[n]-y_n^{(1)}\Big|,
\end{multline}
which leads to
\begin{multline}\label{eq:011}
\mathcal{J}_n^{(1)}\leq C d L (L+1)\Bigg[\sum_{i=1}^d \Big[\omega_1\|W_1(:,i)-d_{1,i}[n]\|_\tF^2\\
+\sum_{l=2}^{L-1}\omega_l
\|W_l(\sigma^{'}(a_{l-1}^{(1)}[n])*d_{l-1,i}[n])-d_{l,i}[n]\|_\tF^2\\
+\sum_{l=2}^{L-1}\omega_l\Big(\sum_{j=l}^{L-1}\|d_{j,i}[n]\|_\tF^2\Big)\|W_l\sigma(a^{(1)}_{l-1}[n])+b_l-a^{(1)}_l[n]\|_\tF^2\\
+\omega_1\Big(\sum_{l=1}^{L-1}\|d_{l,i}[n]\|_\tF^2\Big)\|W_1x_n^{(1)}+b_1-a^{(1)}_1[n]\|_\tF^2\\
+\left|W_L(\sigma^{'}(a_{L-1}^{(1)}[n])*d_{L-1,i}[n])-d_{L,i}[n]\right|^2\Big]
+\Big|\sum_{i=1}^d c_i(x_n^{(1)})d_{L,i}[n]-y_n^{(1)}\Big|^2\Bigg],
\end{multline}
where $C=\max\Big\{1,B_c^2,B_c^2 B_{\sigma^{'}}^{2L-2},B_c^2 C_{\sigma^{'}}^2\max\{B_{\sigma^{'}}^{2L-4},C_{\sigma}^{2L-4}\},
B_c^2 C_{\sigma^{'}}^2\max\{B_{\sigma^{'}}^{2L-6},C_{\sigma}^{2L-6}\}\Big\}$. 
Summing up \eqref{eq:011} from $n=1$ to $N_1$ and dividing by $N_1$ yields 
\begin{equation}\label{eq:012}
    \mathcal{J}^{(1)}\leq C_1 d L (L+1)\JS^{(1)},
\end{equation}
where $C_1=C\max_{l=1,\dots,L}\left\{1,(\alpha_l^{(1)})^{-1},(\beta_l^{(1)})^{-1}\right\}$.

Moreover, using the same proof of Theorem \ref{thm01}, we can obtain
\begin{equation}\label{eq:013}
\begin{split}
&\mathcal{J}^{(2)}\leq C_2 L \JS^{(2)}\leq C_2 dL(L+1) \JS^{(2)},
\end{split}
\end{equation}
where $C_2=\max\{1,C_\sigma^{2L-2}\}\cdot \max_{l=1,\dots,L}\left\{1,(\beta_l^{(2)})^{-1}\right\}$.

Combing \eqref{eq:012} with \eqref{eq:013} yields \eqref{eq:00}.
\end{proof}

In Theorem \ref{thm03}, if the activation satisfies that $B_{\sigma'},C_\sigma,C_{\sigma'}\leq1$ (e.g., $\sigma(\cdot)=\sin(\cdot)$), the constant $\widetilde{C}$ is independent of $L$. Then it implies the relation $\mathcal{J}\leq O(dL^2)\JS$.

On the other hand, $\JS$ is bounded above by $\mathcal{J}$ in the following sense.

\begin{theorem}\label{thm04}
For all $\{W_l,b_l\}_{l=1}^L$, there exists $\{a_l^{(k)}\}_{l=1}^{L-1}$ and $\{d_{l,i}\}_{l=1}^{L-1}$ with $k=1,2$ and $i=1,\ldots,d$ such that
\begin{equation}
\JS\leq\mathcal{J}.
\end{equation}
\end{theorem}

\begin{proof}
It suffices to let $a_1^{(k)} = W_1X^{(k)}+b_1\Bone^\top$, $a_l^{(k)}=W_l\sigma(a_{l-1}^{(k)})+b_l\Bone^\top$, $d_{1,i}=W_1(:,i)\Bone^\top$, and $d_{l,i}=W_l\left(\sigma'(a_{l-1}^{(1)})*d_{l-1,i}\right)$ for $l=2,\dots,L-1$, $k=1,2$, and $i=1,\ldots,d$. In this case the equality holds.
\end{proof}

\subsection{Implementation}\label{sec_implementation_SAPM_PINN}
Similarly to Section \ref{sec_implementation_SAPM_FNN}, we can develop an alternating direction algorithm to minimize $\JS$. Despite its complicated expression, $\JS$ is a quadratic form of each variable $W_l$, $b_l$ and $d_{l,i}$, which allows closed-form minimizers. The derivation is almost the same as that in Section \ref{sec_implementation_SAPM_FNN}, so here we directly give the formulation. For notational simplicity, we assume that the fixed weights $\alpha_{l}^{(1)},\beta_{l}^{(1)},\beta_l^{(2)}$ are all equal to one, and the formulation can easily be generalized to other values. For the sake of brevity, the notations $Z_l$, $\widetilde{A}_l^{(1)}$, $A_l^{(k)}$, $\lambda_l$, $D_l$, $\widetilde{P}_l^{(1)}$, $P_l^{(k)}$, $\theta_l$, $\xi_{n,l}$, $\delta_{n,l}$ appearing in this section are defined in Appendix \ref{Variables}. 

In every iteration, we update $W_l$ by solving the linear systems
\begin{equation}\label{eq:wl}
    \begin{split}
        W_l\begin{bmatrix}\frac{1}{\sqrt{N_1}}Z_l&\frac{1}{\sqrt{N_1}}\widetilde{A}_l^{(1)}&\frac{1}{\sqrt{N_2}}A_l^{(2)}&\sqrt{\lambda_l}I\end{bmatrix}
        =\begin{bmatrix}\frac{1}{\sqrt{N_1}}D_l&\frac{1}{\sqrt{N_1}}\widetilde{P}_l^{(1)} &\frac{1}{\sqrt{N_2}}P_l^{(2)}& O\end{bmatrix},
    \end{split}
\end{equation} 
for $l=2,\ldots, L$, and
\begin{equation}\label{eq:w1}
\begin{split}
    W_1\begin{bmatrix}\frac{1}{\sqrt{N_1}}(\Bone^\top\otimes I_d)&\frac{1}{\sqrt{N_1}}\widetilde{A}_1^{(1)}&\frac{1}{\sqrt{N_2}}A_1^{(2)}\end{bmatrix}
        =\begin{bmatrix}\frac{1}{\sqrt{N_1}}D_1&\frac{1}{\sqrt{N_1}}\widetilde{P}_1^{(1)} &\frac{1}{\sqrt{N_2}}P_1^{(2)}\end{bmatrix},
\end{split}
\end{equation}
for $W_1$. Here $\otimes$ represents the Kronecker product. We update $b_l$ by solving the linear systems
\begin{equation}\label{eq:bl}
b_l\begin{bmatrix}
    \frac{\sqrt{\theta_{l}}}{\sqrt{N_1}}&\frac{1}{\sqrt{N_2}}\Bone^\top
\end{bmatrix}
=\begin{bmatrix}
    \frac{\sqrt{\theta_{l}}}{\sqrt{N_1}}*\Big(a_l^{(1)}-W_lA_l^{(1)}\Big)&
    \frac{1}{\sqrt{N_2}}\Big(a_l^{(2)}-W_lA_l^{(2)}\Big)
\end{bmatrix}
\end{equation}
for $l=1,\ldots, L-1$, and
\begin{equation}\label{eq:bL}
b_L\Bone^\top=Y^{(2)}-W_LA_L^{(2)}.
\end{equation}
for $b_L$. We update $d_{l,i}~(i=1,\dots,d)$ by solving the linear systems
\begin{equation}\label{eq:dli}
    \begin{split}
        \begin{bmatrix}W_{l+1}\diag(\sigma^{'}(a_l^{(1)}[n]))\\ \|W_{l+1}\|_\tF I\\ \sqrt{\xi_{n,l}}I\end{bmatrix}d_{l,i}[n]=\begin{bmatrix}d_{l+1,i}[n]\\\|W_{l+1}\|_\tF D_{l,i}[n]\\ O\end{bmatrix},
    \end{split}
\end{equation}
for $l=1,\dots,L-1$, and
\begin{equation}\label{eq:dLi}
    \begin{split}
        \begin{bmatrix}d_{L,1}[n]~\cdots~d_{L,d}[n]\end{bmatrix}\begin{bmatrix}c_1(x_n^{(1)})& 1&\cdots&0\\
        \vdots& \vdots&\ddots&\vdots\\
        c_d(x_n^{(1)})&0 &\cdots&1
        \end{bmatrix}=\begin{bmatrix}y_n^{(1)} & D_{L,1}[n]&\cdots& D_{L,1}[n]\end{bmatrix},
    \end{split}
\end{equation}
for $d_{L,i}$ $(i=1,\dots,d)$.

Note that the loss function $\JS$ is not a quadratic form of $\{a_l^{(1)}\}$ and $\{a_l^{(2)}\}$, so there are no closed-form minimizers in general. We update $\{a_l^{(1)}\}$, and $\{a_l^{(2)}\}$ by gradient descent. The gradients are computed by
\begin{multline}
\nabla_{a_l^{(1)}[n]}\JS=2\Big[\delta_{n,l}\|W_{l+1}\|_\tF^2(a_l^{(1)}[n]-W_lA_l^{(1)}[n]-b_l)\\
+\delta_{n,l+1}\left(W_{l+1}^\top(W_{l+1}\sigma(a_l^{(1)}[n])-P_{l+1}^{(1)}[n])\right)*\sigma'(a_l^{(1)}[n])\\
+\sum_{i=1}^d\left(\left(W_{l+1}\diag(d_{l,i}[n])\right)^\top\left(W_{l+1}\diag(d_{l,i}[n])\sigma^{'}(a_l^{(1)}[n])-d_{l+1,i}[n]\right)
\right)*\sigma^{''}(a_l^{(1)}[n])\Big],
\end{multline}
and
\begin{equation}
\begin{split}
\nabla_{a_l^{(2)}}\JS=2\Big[\left(W_{l+1}^\top(W_{l+1}\sigma(a_l^{(2)})-P_{l+1}^{(2)})\right)*\sigma'(a_l^{(2)})\\
+\|W_{l+1}\|_\tF^2(a_l^{(2)}-W_lA_l^{(2)}-b_l\Bone^\top)\Big].
\end{split}
\end{equation}

To sum up, we present the algorithm that updates $W_l$, $b_l$, $a_l^{(1)}$, $a_l^{(2)}$, and $d_{l,i}$ alternatively for every $l$ from large to small (see Algorithm \ref{alg02}).

\begin{algorithm}
\DontPrintSemicolon
\KwIn{data $X^{(1)}$, $X^{(2)}$, $Y^{(1)}$, $Y^{(2)}$; number of iterations $N_k$; learning rate $\tau$}
\KwOut{a feasible solution $\{W_l,b_l,a_l^{(1)},a_l^{(2)},d_{l,i}\}$.}
\Begin{initialize $\{W_l,b_l,a_l^{(1)},a_l^{(2)},\{d_{l,i}\}_{i=1}^d\}_{l=1}^L$\;
\For{$k = 0,\cdots,N_k-1$}{
Solve \eqref{eq:wl} for $W_L$\;
Solve \eqref{eq:bL} for $b_L$\;
\For{$n = 1,\cdots,N_1$}{
Solve \eqref{eq:dLi} for $d_{L,i}[n]$ $(i=1,\dots,d)$
}
\For{$l = L-1,\cdots,2$}{
Solve \eqref{eq:wl} for $W_l$. Solve \eqref{eq:bl} for $b_l$\;
\For{$n = 1,\cdots,N_1$}{
\For{$i = 1,\cdots,d$}{
Solve \eqref{eq:dli} for $d_{l,i}[n]$
}
}
\For{$n = 1,\cdots,N_1$}{
$a_l^{(1)}[n]\leftarrow a_l^{(1)}[n]-\tau\nabla_{a_l^{(1)}[n]}\JS$
}
$a_l^{(2)}\leftarrow a_l^{(2)}-\tau\nabla_{a_l^{(2)}}\JS$
}
Solve \eqref{eq:w1} for $W_1$\;
Solve \eqref{eq:bl} for $b_1$\;
\For{$n = 1,\cdots,N_1$}{
\For{$i = 1,\cdots,d$}{
Solve \eqref{eq:dli} for $d_{1,i}[n]$
}
}
\For{$n = 1,\cdots,N_1$}{
$a_1^{(1)}[n]\leftarrow a_1^{(1)}[n]-\tau\nabla_{a_1^{(1)}[n]}\JS$
}
$a_1^{(2)}\leftarrow a_1^{(2)}-\tau\nabla_{a_1^{(2)}}\JS$
}
return $\{W_l,b_l,a_l^{(1)},a_l^{(2)},\{d_{l,i}\}_{i=1}^d\}_{l=1}^L$
}
\caption{Solve $\min_{W_l,b_l,a_l^{(1)},a_l^{(2)},d_{l,i}} \JS$\label{alg02}}
\end{algorithm}

\section{Numerical Experiments}\label{sec_experiments}
In this section, we implement the abovementioned models to solve learning problems and transport equations. These models are summarized in Table \ref{Tab_label}

\begin{table}[h!]\small 
\centering
\begin{tabular}{llcc}
  \toprule
Models & \makecell{Formulation} & Loss & Abbreviation \\\hline
LS-FNN & Eq. \eqref{LS_FNN}& $\mathcal{L}$ & \multirow{2}{*}{LS} \\
LS-PINN & Eq. \eqref{LS_PINN}& $\mathcal{J}$ & \\
PM-FNN & Eq. \eqref{PM_FNN} & $\LP$ & \multirow{2}{*}{PM}\\
PM-PINN & Eq. \eqref{PM_PINN} & $\JP$ & \\
SAPM-FNN & Eq. \eqref{SAPM_FNN} & $\LS$ & \multirow{2}{*}{SAPM}\\
SAPM-PINN & Eq. \eqref{SAPM_PINN} & $\JS$ & \\\bottomrule
\end{tabular}
\caption{\em Models solved in the experiments.}\label{Tab_label}
\end{table}

Details about the experimental settings are listed below. 
\begin{itemize}
  \item {\em Environment for LS.}
  As conventional deep learning models, LS-FNN and LS-PINN are implemented in Python with PyTorch library. The optimization is solved by the vanilla gradient descent method with fine-tuned decaying learning rates. 
  \item {\em Environment for PM and SAPM.}
  SAPM-FNN and SAPM-PINN are solved by Algorithm \ref{alg01} in Section \ref{sec_implementation_SAPM_FNN} and by Algorithm \ref{alg02} in Section \ref{sec_implementation_SAPM_PINN}, respectively. They are implemented in Matlab, where the subroutines {\tt mldivide} (i.e., \textbackslash ) or {\tt mrdivide} (i.e., /) are called to solve the linear least squares systems. For PM-FNN and PM-PINN, we design alternating direction algorithms and implement them in Matlab in a similar way. 
  
  \item {\em Initialization of variables.}
  The network parameters $\{W_l, b_l\}$ and auxiliary variables $\{a_l,a_l^{(1)},a_l^{(2)},d_{l,i}\}$ are randomly initialized with uniform distribution by
  \begin{equation}
  W_l, b_l\sim U(-M^{-1/2},M^{-1/2}),\quad a_l, a_l^{(1)}, a_l^{(2)}, d_{l,i}\sim U(-1,1).
  \end{equation}
  \item {\em Datasets and error evaluation.}
  The training feature vectors $\{x_n\}$ consist of Halton quasirandom points in the problem domain $\Omega$. Then the training dataset is formed by $\{(x_n,y_n)\}$ with $y_n=f(x_n)$.
  
  For the learning problems with FNNs, we compute the $\ell^2$ training error by
  \begin{equation}
  \Eltwo=\left(\frac{\sum_{n=1}^{N}|\phi(x_n;\theta)-y_n|^2}{\sum_{n=1}^{N}|y_n|^2}\right)^\frac{1}{2}.
  \end{equation}
  And we also generate an extra set of Halton points $\{x'_n\}$ in $\Omega$, and let $\{(x'_n,y'_n)\}$ be the testing set, where $y'_n:=f(x'_n)$. Note $\{x'_n\}$ is not overlapped with $\{x_n\}$. We compute the $\ell^2$ testing error $\hEltwo$ in a similar way to evaluate the generalization performance.
  
  For the PDE problems with PINNs, we compute the solution error, i.e. the $\ell^2$ error between the approximate network $\psi$ and the true PDE solution $u$. We generate a testing set of Halton points $\{x'_n\}$, and compute the solution error 
  \begin{equation}
  \hEltwo=\left(\frac{\sum_{n=1}^{N}|\psi(x'_n;\theta)-u(x'_n)|^2}{\sum_{n=1}^{N}|u(x'_n)|^2}\right)^\frac{1}{2}.
  \end{equation}

  \item {\em Randomness.}
  To alleviate the effect of randomness from initialization, we repeat every individual test for 10 times with different seeds. This is equivalent to taking different initial guesses for the optimization. On one hand, we investigate the stability of the models and algorithms by observing the results of all seeds. On the other hand, we focus on the seed that obtains the smallest final loss as the ``best" seed, which shows the best performance under random initialization. 
\end{itemize}

\subsection{Learning Problems with FNNs}
In the first part of the experiments, we learn a function $f$ using FNNs. To test the performance of various network sizes, we implement the algorithms with five depth-width combinations $(L,M)=(6,10)$, $(8,10)$, $(10,10)$, $(6,20)$ and $(6,50)$. The activation function is set as the ReLU activation. All weights $\beta_l$ are set to be 1.
\subsubsection{One-dimensional Case}
In this first example, we set $f(x)=\sin(x^2)$ in the 1-D domain $\Omega=[-1,1]$. The sizes of training and testing set are $10^2$ and $10^3$, respectively. For PM/SAPM, we decrease the actual loss $\LP$/$\LS$ in $5\times10^4$ iterations, obtaining a feasible solution $\{W_l,b_l,a_l\}$. Meanwhile, we also evaluate the mean squared loss $\mathcal{L}$ of the computed $\{W_l,b_l\}$, checking the consistency between $\mathcal{L}$ and $\LP$/$\LS$. For comparison, LS is also implemented with $5\times10^4$ iterations of gradient descent, providing a reference result.

First, let us investigate the learning dynamics of these models. For tests of all seeds, the curves of actual losses versus iterations are plotted in Figure \ref{Fig_Case1_Allloss_iteration}. It is observed that several loss curves of LS get stuck quickly after the start and then remain unchanged over time. This implies that gradient descent on LS will possibly stagnate in bad regions due to vanishing gradient issues and the high non-convexity of the loss landscape. In comparison, PM and SAPM always have decreasing losses during the optimization. Moreover, we can observe that the curves of SAPM are more clustered with lower variance, implying SAPM is more robust than PM over random initialization.

To further quantify the stability over randomness, we report the values of the initial loss before optimization and the final loss after optimization in Table \ref{Tab_Case1_allloss} for all seeds. We say a seed fails in the optimization if the final loss is not below 10\% of the initial loss. It is clear that 9 of 10 seeds fails for $(L,M)=(10,10)$ and 7 of 10 seeds fails for $(L,M)=(6,50)$ in LS. But for PM and SAPM, all seeds obtain significant loss reduction. This reflects the advantage of the alternating direction algorithm over the vanilla gradient descent in the success of random initialization.

Furthermore, for PM and SAPM, we show the curves of their actual losses and the mean squared loss $\mathcal{L}$ versus iterations from the best seed in Figure \ref{Fig_Case1_Bestloss_iteration}. It is seen that the two loss curves of PM are not decreasing synchronously. The optimizer decreases the actual loss curve, but the corresponding mean squared loss $\mathcal{L}$ blows up to a high order of magnitude. Comparatively, the two loss curves of SAPM are decreasing in parallel, consistent with Theorem \ref{thm01} and \ref{thm02}. More precisely, Theorem \ref{thm01} indicates that $\mathcal{L}$ is bounded above by $C_{B,\beta}L\LS$. In this example, $C_{B,\beta}=1$, so $\mathcal{L}$ should be no larger than $L$ times of $\LS$ in theory. Our numerical results show that the relation $\mathcal{L}\leq L\LS$ is always satisfied in SAPM for all seeds. 

Lastly, for PM and SAPM, we report the final actual losses $\LP$/$\LS$ and mean squared losses $\mathcal{L}$ as well as the final errors $\Eltwo$ and $\hEltwo$ of the best seeds in Table \ref{Tab_Case1_bestresult}. We also report the results of LS, whose actual loss is exactly the mean squared losses $\mathcal{L}$. For SAPM, thanks to the consistency between $\mathcal{L}$ and $\LS$, $\mathcal{L}$ is decreased together with $\LS$ by the optimizer; the training error $\Eltwo$ is also decreased together with $\mathcal{L}$, noting that $\Eltwo=C\mathcal{L}^{1/2}$ for some constant $C$ independent of the variables. The final errors of SAPM have lower or equal orders of magnitude than LS (recall that the reported final errors of LS are chosen from the best seeds, yet most seeds fail in the optimization and obtain very large errors). Comparatively, PM outputs large $\mathcal{L}$ and errors, though its actual loss can be reduced to a satisfactorily low level by the optimizer.

We remark that for PM and SAPM, bad regularization may appear in local regions. For example, in the case $(L,M)=(6,10)$, the learner $\phi(x,\theta)$ obtained from the sixth seed is shown in Figure \ref{Fig_Case1_generalization}, as well as the true target function $f(x)$. It is observed that $\phi$ has an erroneous spike in a small region near $x=0.92$, where no training points are located. The local bad regularization possibly comes from the special algorithms that update variables alternatively. It is different from the usual case in deep learning that all variables are updated simultaneously via gradient descent, which guarantees the regularization of the learner network implicitly \cite{Neyshabur2017,Lei2018,Cao2021}.

\begin{figure}
\centering
\includegraphics[scale=0.7]{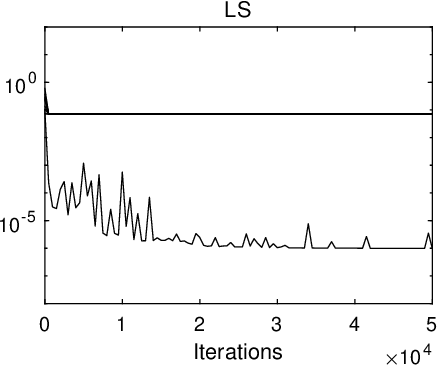}\hspace{30pt}
\includegraphics[scale=0.7]{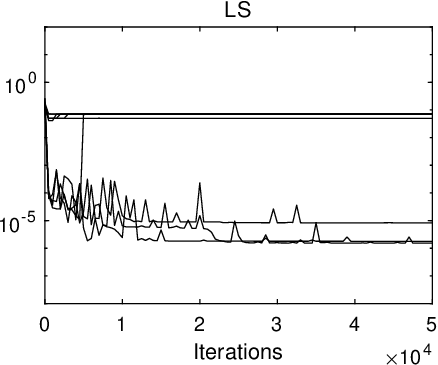}\\\vspace{10pt}
\includegraphics[scale=0.7]{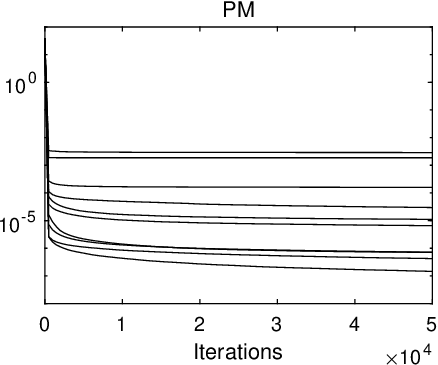}\hspace{30pt}
\includegraphics[scale=0.7]{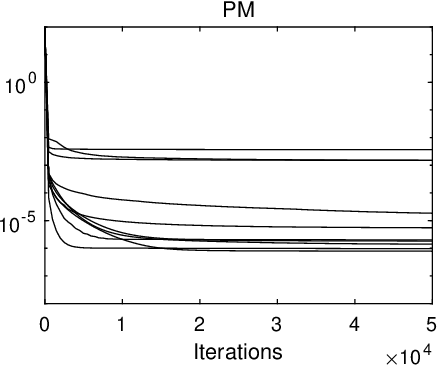}\\
\subfloat[$L=10,M=10$]{
\includegraphics[scale=0.7]{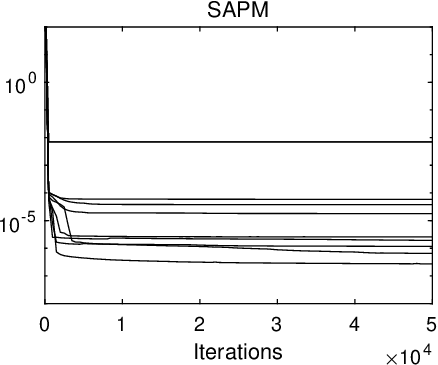}}\hspace{30pt}
\subfloat[$L=6,M=50$]{
\includegraphics[scale=0.7]{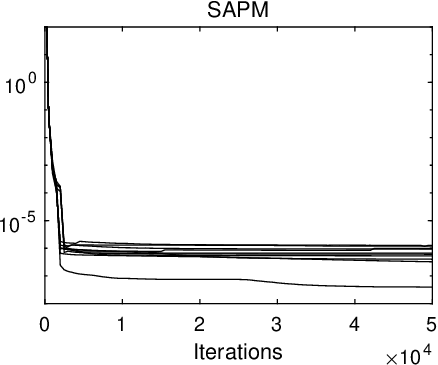}}
\caption{\em The loss versus iterations in learning $f(x)=\sin(x^2)$. (Results are from 10 random seeds.)}
\label{Fig_Case1_Allloss_iteration}
\end{figure}

\begin{table}
\fontsize{7}{9.6}\selectfont
\centering
\subfloat[$L=10,M=10$]{
\begin{tabular}{llll}
  \toprule
Seed  & \multicolumn{1}{c}{LS} & \multicolumn{1}{c}{PM} & \multicolumn{1}{c}{SAPM} \\\hline
1& 6.25e-01 $\rightarrow$ 7.05e-02* & 3.75e+01 $\rightarrow$ 1.45e-07 & 4.22e+04 $\rightarrow$ 6.49e-07 \\
2& 1.09e-01 $\rightarrow$ 7.05e-02* & 3.87e+01 $\rightarrow$ 2.89e-03 & 2.37e+04 $\rightarrow$ 1.92e-06 \\
3& 3.92e-01 $\rightarrow$ 7.05e-02* & 3.76e+01 $\rightarrow$ 7.04e-07 & 1.44e+04 $\rightarrow$ 5.86e-05 \\
4& 2.38e-01 $\rightarrow$ 7.05e-02* & 3.90e+01 $\rightarrow$ 1.11e-05 & 3.28e+04 $\rightarrow$ 3.77e-05 \\
5& 1.41e-01 $\rightarrow$ 7.05e-02* & 3.93e+01 $\rightarrow$ 2.96e-05 & 2.68e+04 $\rightarrow$ 7.02e-03 \\
6& 2.64e-01 $\rightarrow$ 7.05e-02* & 3.75e+01 $\rightarrow$ 1.87e-03 & 3.05e+04 $\rightarrow$ 1.17e-06 \\
7& 2.99e-01 $\rightarrow$ 7.05e-02* & 3.76e+01 $\rightarrow$ 6.55e-06 & 2.73e+04 $\rightarrow$ 1.79e-05 \\
8& 7.49e-02 $\rightarrow$ 7.05e-02* & 3.77e+01 $\rightarrow$ 1.59e-04 & 2.73e+04 $\rightarrow$ 7.01e-03 \\
9& 1.04e-01 $\rightarrow$ \textbf{9.80e-07} & 3.72e+01 $\rightarrow$ 7.22e-07 & 1.95e+04 $\rightarrow$ 2.72e-07 \\
10& 7.28e-02 $\rightarrow$ 7.05e-02* & 3.71e+01 $\rightarrow$ 4.22e-07 & 1.01e+04 $\rightarrow$ 2.57e-06 \\
\bottomrule
\end{tabular}}\\
\subfloat[$L=6,M=50$]{
\begin{tabular}{llll}
  \toprule
Seed  & \multicolumn{1}{c}{LS} & \multicolumn{1}{c}{PM} & \multicolumn{1}{c}{SAPM} \\\hline
1& 2.68e-01 $\rightarrow$ 1.76e-06 & 9.74e+01 $\rightarrow$ 2.00e-06 & 6.09e+05 $\rightarrow$ 3.24e-07 \\
2& 1.31e-01 $\rightarrow$ 7.05e-02* & 9.65e+01 $\rightarrow$ 1.82e-06 & 4.64e+05 $\rightarrow$ 1.27e-06 \\
3& 2.50e-01 $\rightarrow$ 8.32e-06 & 9.66e+01 $\rightarrow$ 1.50e-03 & 5.17e+05 $\rightarrow$ 4.00e-07 \\
4& 1.81e-01 $\rightarrow$ 7.05e-02* & 9.66e+01 $\rightarrow$ 9.64e-07 & 5.66e+05 $\rightarrow$ 6.61e-07 \\
5& 1.21e-01 $\rightarrow$ 7.05e-02* & 9.70e+01 $\rightarrow$ 1.51e-03 & 5.40e+05 $\rightarrow$ 9.32e-07 \\
6& 1.54e-01 $\rightarrow$ 4.95e-02* & 9.56e+01 $\rightarrow$ 5.49e-06 & 5.08e+05 $\rightarrow$ 1.20e-06 \\
7& 2.00e-01 $\rightarrow$ \textbf{1.52e-06} & 9.79e+01 $\rightarrow$ 3.67e-03 & 5.69e+05 $\rightarrow$ 6.06e-07 \\
8& 1.07e-01 $\rightarrow$ 7.05e-02* & 9.47e+01 $\rightarrow$ 1.84e-05 & 4.48e+05 $\rightarrow$ 3.94e-08 \\
9& 1.47e-01 $\rightarrow$ 7.05e-02* & 9.63e+01 $\rightarrow$ 7.79e-07 & 4.56e+05 $\rightarrow$ 5.31e-07 \\
10& 9.71e-02 $\rightarrow$ 7.05e-02* & 9.85e+01 $\rightarrow$ 1.40e-06 & 4.77e+05 $\rightarrow$ 9.76e-07 \\
\bottomrule
\end{tabular}}
\caption{\em Initial loss $\rightarrow$ final loss of the model optimization from Seed 1 to Seed 10 in learning $f(x)=sin(x^2)$. (Results with asterisk mean failed optimization that the final loss is no less than 10\% if the initial loss.)}
\label{Tab_Case1_allloss}
\end{table}

\begin{figure}
\centering
\subfloat[$L=10,M=10$]{
\includegraphics[scale=0.7]{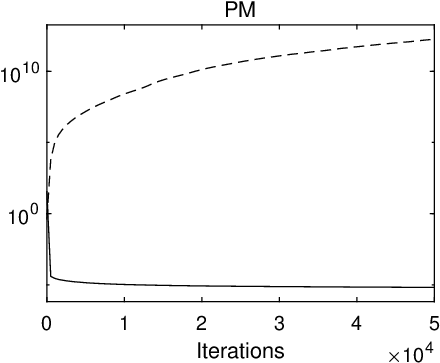}\hspace{30pt}
\includegraphics[scale=0.7]{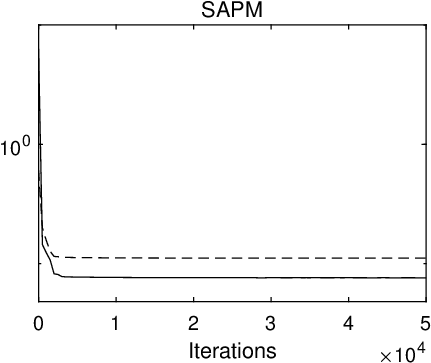}}\\
\subfloat[$L=6,M=50$]{
\includegraphics[scale=0.7]{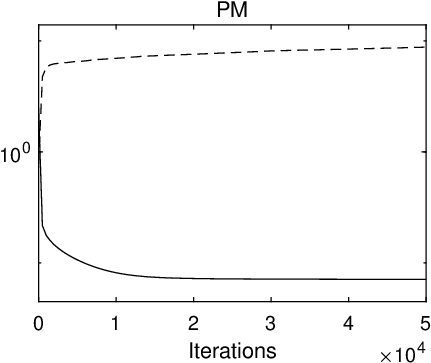}\hspace{30pt}
\includegraphics[scale=0.7]{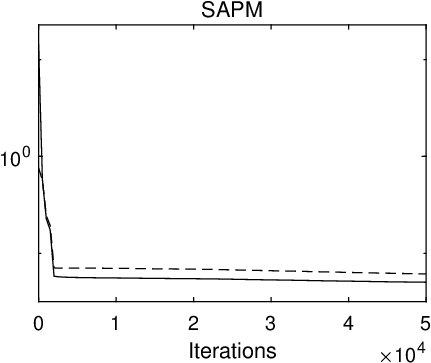}}
\caption{\em The actual loss $\LP$/$\LS$ (solid curve) and the corresponding mean squared loss $\mathcal{L}$ (dashed curve) versus iterations in learning $f(x)=\sin(x^2)$. (Results are from the best seed.)}
\label{Fig_Case1_Bestloss_iteration}
\end{figure}

\begin{table}
\fontsize{7}{9.6}\selectfont
\centering
\begin{tabular}{clcccc}
  \toprule
$(L,M)$ & Models & Actual loss & $\mathcal{L}$ & $\Eltwo$ & $\hEltwo$ \\\hline
\multirow{3}{*}{$(6,10)$} & LS & \multicolumn{2}{c}{9.76e-06} & 7.64e-03 & 7.89e-03 \\
& PM & 4.10e-07 & 1.03e+02 & 2.48e+01 & 2.44e+01 \\
& SAPM & 4.41e-07 & 1.67e-06 & 3.16e-03 & 3.28e-03 \\\hline
\multirow{3}{*}{$(8,10)$} & LS & \multicolumn{2}{c}{2.04e-06} & 3.49e-03 & 3.77e-03 \\
& PM & 1.81e-06 & 4.02e+02 & 4.90e+01 & 4.12e+01 \\
& SAPM & 8.74e-07 & 5.58e-06 & 5.78e-03 & 7.31e-03 \\\hline
\multirow{3}{*}{$(10,10)$} & LS & \multicolumn{2}{c}{9.80e-07} & 2.42e-03 & 3.19e-03 \\
& PM & 1.45e-07 & 5.97e-03 & 1.89e-01 & 1.88e-01 \\
& SAPM & 2.72e-07 & 1.84e-06 & 3.32e-03 & 3.42e-03 \\\hline
\multirow{3}{*}{$(6,20)$} & LS & \multicolumn{2}{c}{3.43e-07} & 1.43e-03 & 2.74e-03 \\
& PM & 5.88e-06 & 6.91e+05 & 2.03e+03 & 1.93e+03 \\
& SAPM & 5.62e-07 & 1.60e-06 & 3.09e-03 & 4.86e-03 \\\hline
\multirow{3}{*}{$(6,50)$} & LS & \multicolumn{2}{c}{1.52e-06} & 3.01e-03 & 3.18e-03 \\
& PM & 7.79e-07 & 9.13e-01 & 2.34e+00 & 6.10e+01 \\
& SAPM & 3.94e-08 & 7.99e-08 & 6.91e-04 & 9.94e-03 \\
\bottomrule
\end{tabular}
\captionsetup{width=\textwidth}
\caption{\em The final actual losses, associated mean squared losses $\mathcal{L}$, training errors $\Eltwo$ and testing errors $\hEltwo$ in learning $f(x)=sin(x^2)$. (For each $(L,M)$ combination, the result is from the best seed. Note that most seeds of LS can not achieve these good results.)}
\label{Tab_Case1_bestresult}
\end{table}

\begin{figure}
\centering
\includegraphics[scale=0.7]{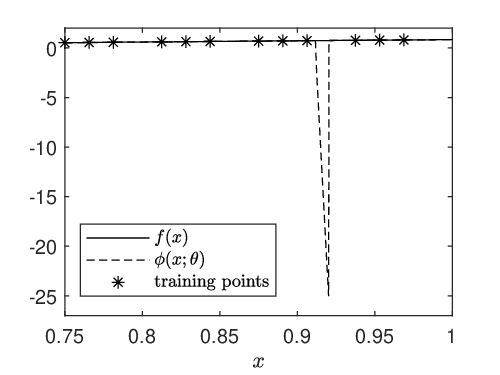}
\caption{\em Partial plots of the learner $\phi(x;\theta)$, target function $f(x)=\sin(x^2)$ and training points $\{(x_n,y_n)\}$ in a test of SAPM.}
\label{Fig_Case1_generalization}
\end{figure}

\subsubsection{High-dimensional case}
In this second example, we set $f(x)=1/(2\sqrt{d}+\sum_{i=1}^dx_i)$ in the unit ball $\Omega=\{x\in\mathbb{R}^d:\|x\|_2\leq1\}$. We set the dimension $d=10$. To address the high dimension, we enlarge the size of the training dataset to $10^4$. The size combinations $(L,M)=(6,10),(8,10),(10,10)$ are tested.

Phenomena similar to the preceding case are observed. To be precise, we report the initial and final losses of the models in Table \ref{Tab_Case2_allloss}. It is seen that most seeds of LS can only slightly reduce the loss function: for $L=6$, only the tenth seed reduces the loss to $O(10^{-8})$, and the others are $O(10^{-4})$; For $L=8,10$, all 10 seeds obtain poor optimization with final losses $O(10^{-4})$. This means it is difficult to optimize LS by gradient descent. Comparatively, PM and SAPM have good optimization results: their losses are always reduced to a large extent, e.g., from $10^{1}$ to $10^{-8}$. Therefore, it implies that the optimization of PM and SAPM can be solved more effectively then LS using the alternating direction algorithms.

Moreover, SAPM has simultaneous decreases of $\mathcal{L}$ and $\LS$, but PM does not; therefore, SAPM achieves more accurate learning than PM. We report the final results from the best seed in Table \ref{Tab_Case2_bestresult}. SAPM can always obtain $O(10^{-3})$ errors, but PM performs much worse for $L=6,8$ with $O(10^{-1})$ errors and completely fails for $L=10$ with $O(10^3)$ errors. Although LS obtains smaller error $O(10^{-4})$ when $L=6$, this result is from the unique seed that does not fail (the tenth seed in Table \ref{Tab_Case2_allloss}); all the other seeds have low-accuracy results, just as the case $L=8,10$ where LS only gets $O(10^{-1})$ errors. 

\begin{table}[h!]
\fontsize{7}{9.6}\selectfont
\centering
\subfloat[$L=6,M=10$]{
\begin{tabular}{llll}
  \toprule
Seed  & \multicolumn{1}{c}{LS} & \multicolumn{1}{c}{PM} & \multicolumn{1}{c}{SAPM} \\\hline
1& 2.71e-01 $\rightarrow$ 6.05e-04 & 2.04e+01 $\rightarrow$ 1.46e-08 & 2.26e+02 $\rightarrow$ 1.08e-06 \\
2& 6.75e-03 $\rightarrow$ 6.10e-04 & 2.07e+01 $\rightarrow$ 9.39e-07 & 2.97e+02 $\rightarrow$ 1.96e-07 \\
3& 1.79e-01 $\rightarrow$ 6.30e-04 & 2.09e+01 $\rightarrow$ 3.21e-08 & 1.70e+02 $\rightarrow$ 5.01e-06 \\
4& 2.02e-02 $\rightarrow$ 6.08e-04 & 2.01e+01 $\rightarrow$ 1.11e-07 & 2.07e+02 $\rightarrow$ 2.63e-06 \\
5& 2.16e-03 $\rightarrow$ 6.02e-04 & 2.11e+01 $\rightarrow$ 4.62e-09 & 2.11e+02 $\rightarrow$ 7.05e-06 \\
6& 1.27e-01 $\rightarrow$ 6.05e-04 & 2.13e+01 $\rightarrow$ 2.90e-07 & 2.93e+02 $\rightarrow$ 1.24e-05 \\
7& 5.92e-02 $\rightarrow$ 5.94e-04 & 2.16e+01 $\rightarrow$ 1.70e-08 & 2.34e+02 $\rightarrow$ 1.31e-06 \\
8& 1.53e-03 $\rightarrow$ 6.10e-04 & 2.08e+01 $\rightarrow$ 4.89e-08 & 3.03e+02 $\rightarrow$ 1.04e-05 \\
9& 3.46e-03 $\rightarrow$ 6.19e-04 & 2.12e+01 $\rightarrow$ 1.60e-07 & 3.32e+02 $\rightarrow$ 2.60e-06 \\
10& 2.82e-02 $\rightarrow$ \textbf{1.62e-08} & 2.13e+01 $\rightarrow$ 9.62e-09 & 2.52e+02 $\rightarrow$ 8.92e-08 \\
\bottomrule
\end{tabular}}\\
\subfloat[$L=8,M=10$]{
\begin{tabular}{llll}
  \toprule
Seed  & \multicolumn{1}{c}{LS} & \multicolumn{1}{c}{PM} & \multicolumn{1}{c}{SAPM} \\\hline
1& 7.75e-04 $\rightarrow$ 6.04e-04 & 2.93e+01 $\rightarrow$ 3.26e-08 & 3.03e+03 $\rightarrow$ 6.64e-06 \\
2& 1.20e-01 $\rightarrow$ 6.09e-04 & 3.00e+01 $\rightarrow$ 7.13e-07 & 3.44e+03 $\rightarrow$ 6.39e-07 \\
3& 4.66e-03 $\rightarrow$ 5.98e-04 & 2.91e+01 $\rightarrow$ 2.63e-09 & 2.09e+03 $\rightarrow$ 4.49e-06 \\
4& 1.50e-03 $\rightarrow$ \textbf{5.90e-04} & 2.94e+01 $\rightarrow$ 1.16e-08 & 2.47e+03 $\rightarrow$ 3.85e-07 \\
5& 2.13e-01 $\rightarrow$ 6.08e-04 & 2.99e+01 $\rightarrow$ 1.46e-08 & 3.05e+03 $\rightarrow$ 3.87e-07 \\
6& 1.77e-01 $\rightarrow$ 6.07e-04 & 2.98e+01 $\rightarrow$ 8.36e-08 & 3.42e+03 $\rightarrow$ 3.70e-06 \\
7& 8.22e-03 $\rightarrow$ 6.12e-04 & 2.93e+01 $\rightarrow$ 1.51e-07 & 2.09e+03 $\rightarrow$ 9.06e-08 \\
8& 2.40e-01 $\rightarrow$ 5.94e-04 & 2.87e+01 $\rightarrow$ 6.58e-08 & 3.09e+03 $\rightarrow$ 1.07e-06 \\
9& 1.06e-01 $\rightarrow$ 6.03e-04 & 2.99e+01 $\rightarrow$ 9.46e-07 & 4.09e+03 $\rightarrow$ 8.75e-07 \\
10& 5.47e-03 $\rightarrow$ 6.01e-04 & 2.97e+01 $\rightarrow$ 3.89e-09 & 2.59e+03 $\rightarrow$ 1.16e-07 \\
\bottomrule
\end{tabular}}\\
\subfloat[$L=10,M=10$]{
\begin{tabular}{llll}
  \toprule
Seed  & \multicolumn{1}{c}{LS} & \multicolumn{1}{c}{PM} & \multicolumn{1}{c}{SAPM} \\\hline
1& 1.33e-02 $\rightarrow$ 6.06e-04 & 3.79e+01 $\rightarrow$ 2.35e-08 & 3.17e+04 $\rightarrow$ 9.41e-07 \\
2& 4.39e-02 $\rightarrow$ 6.05e-04 & 3.88e+01 $\rightarrow$ 6.13e-08 & 4.32e+04 $\rightarrow$ 5.27e-07 \\
3& 4.54e-03 $\rightarrow$ \textbf{5.95e-04} & 3.68e+01 $\rightarrow$ 2.52e-08 & 2.00e+04 $\rightarrow$ 1.03e-06 \\
4& 1.45e-01 $\rightarrow$ 6.07e-04 & 3.77e+01 $\rightarrow$ 6.57e-08 & 2.97e+04 $\rightarrow$ 5.97e-07 \\
5& 2.36e-02 $\rightarrow$ 6.17e-04 & 3.82e+01 $\rightarrow$ 1.10e-08 & 4.72e+04 $\rightarrow$ 1.45e-06 \\
6& 1.77e-01 $\rightarrow$ 6.12e-04 & 3.91e+01 $\rightarrow$ 6.99e-08 & 3.21e+04 $\rightarrow$ 2.79e-07 \\
7& 1.72e-02 $\rightarrow$ 5.95e-04 & 3.73e+01 $\rightarrow$ 1.61e-08 & 2.43e+04 $\rightarrow$ 2.65e-07 \\
8& 1.28e-02 $\rightarrow$ 6.17e-04 & 3.66e+01 $\rightarrow$ 3.23e-08 & 2.92e+04 $\rightarrow$ 6.18e-07 \\
9& 7.47e-02 $\rightarrow$ 6.19e-04 & 3.76e+01 $\rightarrow$ 1.51e-07 & 3.76e+04 $\rightarrow$ 3.40e-07 \\
10& 1.18e-02 $\rightarrow$ 5.96e-04 & 3.77e+01 $\rightarrow$ 5.33e-08 & 3.66e+04 $\rightarrow$ 2.85e-08 \\
\bottomrule
\end{tabular}}
\caption{\em Initial loss $\rightarrow$ final loss of the model optimization from Seed 1 to Seed 10 in learning the 10-dimensional function $f(x)=1/(2\sqrt{d}+\sum_{i=1}^dx_i)$.}
\label{Tab_Case2_allloss}
\end{table}

\begin{table}
\fontsize{7}{9.6}\selectfont
\centering
\begin{tabular}{clcccc}
  \toprule
$(L,M)$ & Models & Actual loss & $\mathcal{L}$ & $\Eltwo$ & $\hEltwo$ \\\hline
\multirow{3}{*}{$(6,10)$} & LS & \multicolumn{2}{c}{1.62e-08} & 7.78e-04 & 9.37e-04 \\
& PM & 4.62e-09 & 1.10e-03 & 2.03e-01 & 2.06e-01 \\
& SAPM & 8.92e-08 & 3.84e-07 & 3.79e-03 & 4.01e-03 \\\hline
\multirow{3}{*}{$(8,10)$} & LS & \multicolumn{2}{c}{5.90e-04} & 1.49e-01 & 1.51e-01 \\
& PM & 2.63e-09 & 6.50e-04 & 1.56e-01 & 1.57e-01 \\
& SAPM & 9.06e-08 & 4.61e-07 & 4.15e-03 & 4.23e-03 \\\hline
\multirow{3}{*}{$(10,10)$} & LS & \multicolumn{2}{c}{5.95e-04} & 1.49e-01 & 1.51e-01 \\
& PM & 1.10e-08 & 9.93e+04 & 1.93e+03 & 1.88e+03 \\
& SAPM & 2.85e-08 & 1.71e-07 & 2.53e-03 & 2.12e-03 \\
\bottomrule
\end{tabular}
\captionsetup{width=\textwidth}
\caption{\em The final actual losses, associated mean squared losses $\mathcal{L}$, training errors $\Eltwo$ and testing errors $\hEltwo$ in learning the 10-dimensional function $f(x)=1/(2\sqrt{d}+\sum_{i=1}^dx_i)$. (Results are from the best seed.)}
\label{Tab_Case2_bestresult}
\end{table} 

\subsection{Transport Equations}
In the second part, we solve the transport equations 
\begin{gather}
\frac{\partial u}{\partial t}(t,x)-v(t,x)\cdot\nabla_{x}u(t,x)=f(t,x),~~ \text{for}~ t\in (0,T],~x\in\Omega,\nonumber\\
u(t,x)=g(t,x),~~ \text{for}~ t\in (0,T],~x\in\Gamma,\label{eq:Trans.Eq}\\
u(0,x)=u_0(x),~~ \text{for}~ x\in\Omega,\nonumber
\end{gather}
using LS-PINN, PM-PINN and SAPM-PINN. Various network sizes are tested. We use activation $\sin(\cdot)$ in these experiments. The scalar weights $\alpha_l^{(1)}$, $\beta_l^{(1)}$ and $\beta_l^{(2)}$ are set to be 1.

\subsubsection{One-dimensional Problem}
In the first example, we let $\Omega=[-1,1]$ be a 1-D interval and $v(t,x)=t+x+\frac 32$. The true solution is set as $$u(t,x)=e^t\sin x.$$

We set the numbers of sample points $\{x_n^{(1)}\}$ and $\{x_n^{(2)}\}$ as $N_1=1000$ and $N_2=400$. The depth-width combinations $(L,M)=(6,10)$, $(8,20)$ and $(10,50)$ are tested, and the number of iterations is set as $10^3$. The actual losses, mean squared losses and solution errors calculated from the best seeds are presented in Table \ref{Tab_Case3_bestresult}. It shows that when $L=6$, the solution errors obtained through LS and SAPM can achieve $O(10^{-3})$, but the solution error obtained by PM is only $O(10^{-1})$. For larger network size $(8,20)$, the solution error of LS remains $O(10^{-3})$, while the error of SAPM decreases to $O(10^{-4})$; PM fails with an error close to 1. Furthermore, for the network size $(10, 50)$, SAPM is still effective with acceptable accuracy $O(10^{-3})$; in comparison, LS and PM completely fail with errors close to 1.

In addition, according to Theorem \ref{thm03} and \ref{thm04}, the actual loss $\JS$ and the mean squared loss $\mathcal{J}$ of SAPM are consistent. More specifically, Theorem \ref{thm03} indicates that $\mathcal{J}$ is bounded above by $\widetilde{C}dL(L+1)\JS$. In this example, $\widetilde{C}=1$, so $\mathcal{J}$ should not exceed $dL(L+1)$ times $\JS$. It is verified that this relation holds in every iteration of the SAPM algorithm in our numerical experiments.

\begin{table}
\fontsize{7}{9.6}\selectfont
\centering
\begin{tabular}{clccc}
  \toprule
$(L,M)$ & Models & Actual loss & $\mathcal{J}$ & $\hEltwo$ \\\hline
\multirow{3}{*}{$(6,10)$} & LS & \multicolumn{2}{c}{1.22e-04} & 3.99e-03 \\
& PM & 6.64e-06 & 2.20e+00 & 4.52e-01 \\
& SAPM & 5.18e-05 & 8.73e-04 & 6.85e-03 \\\hline
\multirow{3}{*}{$(8,20)$} & LS & \multicolumn{2}{c}{1.75e-04} & 6.16e-03 \\
& PM & 1.12e-03 & 1.07e+01 & 9.53e-01 \\
& SAPM & 3.36e-07 & 4.79e-06 & 6.92e-04 \\\hline
\multirow{3}{*}{$(10,50)$} & LS & \multicolumn{2}{c}{2.24e+00} & 1.00e+00 \\
& PM & 2.62e-03 & 1.11e+01 & 9.65e-01 \\
& SAPM & 4.05e-07 & 3.35e-06 & 1.56e-03 \\
\bottomrule
\end{tabular}
\captionsetup{width=\textwidth}
\caption{\em The final actual losses, associated mean squared losses $\mathcal{J}$ and solution errors $\hEltwo$ in solving the one-dimensional transport equation. (Results are from the best seed.)}
\label{Tab_Case3_bestresult}
\end{table}  

\subsubsection{Three-dimensional Problem}
In the second example, we solve the problem \eqref{eq:Trans.Eq} with $v_i(t,x)=x_i+2$ and the true solution $$u(t,x)=\sum_{i=1}^d(t+x_i)\sin x_i,$$
where $d=3$ and $\Omega=[-1,1]^3$. Here, we choose $N_1=4000$ and $N_2=1600$ points to form the training dataset. We implement the algorithms for three depth-width combinations $(L,M)=(6,40)$, $(8,30)$ and $(8,40)$ with $10^3$ iterations. The results are listed in Table \ref{Tab_Case4_bestresult}. 

For all network sizes, SAPM achieves errors around $O(10^{-2})$, but LS and PM obtain errors as large as $O(10^{-1})$. For LS, as the network size increases, the final loss rises, implying the difficulty in optimizing wider and deeper networks without auxiliary variables. For PM, despite its small final losses around $O(10^{-3})$, the final solution error is large, which reflects the inconsistency between the loss function and solution error.   

In addition, the numerical results of this experiment still verify Theorem \ref{thm03} and \ref{thm04}. The actual loss $\JS$ and the mean squared loss $\mathcal{J}$ of SAPM always decrease in parallel, with the latter bounded above by $dL(L+1)$ times the former.

\begin{table}
\fontsize{7}{9.6}\selectfont
\centering
\begin{tabular}{clccc}
\toprule
$(L,M)$ & Models & Actual loss & $\mathcal{J}$ & $\hEltwo$ \\\hline
\multirow{3}{*}{$(6,40)$} & LS & \multicolumn{2}{c}{1.13e-01} & 1.91e-01 \\
& PM & 1.97e-03 & 1.53e+01 & 5.16e-01 \\
& SAPM & 1.51e-03 & 4.60e-02 & 3.25e-02 \\\hline
\multirow{3}{*}{$(8,30)$} & LS & \multicolumn{2}{c}{1.38e-01} & 2.52e-01 \\
& PM & 2.90e-03 & 2.72e+01 & 6.29e-01 \\
& SAPM & 2.51e-03 & 7.29e-02 & 3.72e-02 \\\hline
\multirow{3}{*}{$(8,40)$} & LS & \multicolumn{2}{c}{4.19e+00} & 6.25e-01 \\
& PM & 1.53e-03 & 2.55e+01 & 6.12e-01 \\
& SAPM & 1.56e-03 & 4.75e-02 & 3.29e-02 \\
\bottomrule
\end{tabular}
\captionsetup{width=\textwidth}
\caption{\em The final actual losses, associated mean squared losses $\mathcal{J}$ and solution errors $\hEltwo$ in solving the three-dimensional transport equation. (Results are from the best seed.)}
\label{Tab_Case4_bestresult}
\end{table} 

\section{Conclusion}\label{sec5}
This work considers effective models for least squares deep learning. Due to the high non-convexity of the mean squared loss with deep neural networks, common optimizers are usually less effective in finding good minima. Instead of directly optimizing the mean squared loss, one can resort to alternative models. One technique is introducing auxiliary variables to split the deep neural network layer by layer, and these variables are governed by penalty terms added to the loss function. This penalty model (PM) enhances the convexity of the loss function with variables but does not preserve the consistency between its loss function and the original mean squared loss; hence, it may have low accuracy in practice. 

To overcome the issue, we develop novel models with self-adaptive weighted auxiliary variables (SAPM) for learning problems with FNNs and first-order linear PDEs with PINNs. Theoretically, we prove that the loss function of SAPM is consistent with the original mean squared loss. Moreover, we design alternating direction algorithms to solve SAPM efficiently. In numerical experiments, we compare SAPM with the original least squares model (LS) and PM. Results show that the proposed SAPM has higher accuracy than LS and PM; in particular, for deeper networks, SAPM still performs correctly, but LS and PM completely fail. Numerical results also verify the proven consistency that the mean squared loss is always bound above by the SAPM loss function. 

Future work could consider the application of the self-adaptive auxiliary variables to other types of deep learning problems, such as PINN models from high-order PDEs, image processing by convolutional neural networks, and time-dependent problems by recurrent neural networks.
\appendix 
\section{Notations used in Section \ref{sec_implementation_SAPM_PINN}}\label{Variables}
The notations used in Section \ref{sec_implementation_SAPM_PINN} are defined as follows.
\begin{align*}
Z_l:=&\begin{bmatrix}Z_{l,1}[1]~\dots~Z_{l,1}[N_1]~\dots~Z_{l,d}[1]~\dots~Z_{l,d}[N_1]\end{bmatrix},  \quad 2\leq l\leq L;\\
\widetilde{A}_l^{(1)}:=&\begin{bmatrix}\sqrt{\gamma_{1,l}}A_l^{(1)}[1]~\dots~\sqrt{\gamma_{N_1,l}}A_l^{(1)}[N_1]\end{bmatrix}, \quad 1\leq l\leq L;\\
A_l^{(k)}:=&\begin{cases}\sigma(a_{l-1}^{(k)}),& 2\leq l\leq L,\\
X^{(k)},& l=1,\end{cases}\quad k=1,2;\\
\lambda_l:=&\frac{1}{N_1}\Big[\sum_{i=1}^d\sum_{m=1}^{l-1}\widetilde{\omega}_m\|D_{m,i}-d_{m,i}\|_\tF^2+\eta_l\sum_{m=1}^{l-1}\widetilde{\omega}_m\|(W_mA_m^{(1)}-P_m^{(1)})\Omega_l\|_\tF^2\\
&+\sum_{m=1}^{l-1}\widetilde{\omega}_m\|(W_mA_m^{(1)}-P_m^{(1)})\widetilde{\Omega}_m\|_\tF^2\Big]+\frac{1}{N_2}\sum_{m=1}^{l-1}\widetilde{\omega}_m\|W_m A_m^{(2)}-P_m^{(2)}\|_\tF^2,\\
&\qquad\qquad\qquad\qquad\qquad\qquad\qquad\qquad\qquad\qquad\qquad\qquad\qquad 3\leq l\leq L;\\
\lambda_2:=&\frac{1}{N_1}\Big[\sum_{i=1}^d\|D_{1,i}-d_{1,i}\|_\tF^2+\|(W_{1}A_{1}^{(1)}-P_{1}^{(1)})\Omega_1\|_\tF^2\Big]+\frac{1}{N_2}\|W_1A_1^{(2)}-P_1^{(2)}\|_\tF^2;\\
D_l:=&\begin{bmatrix}d_{l,1}[1]~\dots~ d_{l,1}[N_1] ~\dots~ d_{l,d}[1] ~\dots~ d_{l,d}[N_1]\end{bmatrix}, \quad 1\leq l\leq L;\\
\widetilde{P}_l^{(1)}:=&\begin{bmatrix}\sqrt{\gamma_{1,l}}P_l^{(1)}[1]~\dots~\sqrt{\gamma_{N_1,l}}P_l^{(1)}[N_1]\end{bmatrix}, \quad 1\leq l\leq L;\\
P_l^{(k)}:=&\begin{cases}Y^{(k)}-b_L\Bone^\top,& l=L,\\
a_l^{(k)}-b_l\Bone^\top,& 1\leq l\leq L-1,\end{cases}\quad k=1,2;\\
\theta_{l}:=&\Big[\sum_{k=l}^{L-1}\sum_{i=1}^d\|d_{l,i}[1]\|_\tF^2 ~\dots~ \sum_{k=l}^{L-1}\sum_{i=1}^d\|d_{l,i}[N_1]\|_\tF^2\Big],\quad 1\leq l\leq L-1;\\
\xi_{n,l}:=&\sum_{j=1}^l\Big(\prod_{k=j+1}^{l+1}\|W_k\|_\tF^2\Big)\|W_jA_j^{(1)}[n]-P_j^{(1)}[n]\|_\tF^2, \quad 1\leq l\leq L-1;\\
\delta_{n,l}:=&\begin{cases}\sum_{m=l}^{L-1}\|d_{m,i}[n]\|_\tF^2, & 1\leq l\leq L-1,\\ 0,& l = L,\end{cases}
\end{align*}
with
\begin{align*}
Z_{l,i}:=&\sigma^{'}(a_{l-1}^{(1)})*d_{l-1,i},\quad 2\leq l\leq L,~1\leq i\leq d;\\
\gamma_{n,l}:=&\begin{cases}\sum_{i=1}^{d}\sum_{m=l}^{L-1}\|d_{m,i}[n]\|_\tF^2, & 1\leq l\leq L-1,\\ 0,& l = L,\end{cases}\quad 1\leq n\leq N_1;\\
\widetilde{\omega}_m:=&\begin{cases}
    \prod_{k=m+1}^{l-1}\|W_k\|_\tF^2,& 1\leq m\leq l-2,\\
    1,& m=l-1,
\end{cases}\quad 3\leq l\leq L;\\
\widetilde{\Omega}_m:=&\diag\left(\sqrt{\sum_{i=1}^{d}\sum_{j=m}^{l-1}\|d_{j,i}[1]\|_\tF^2}~\dots~\sqrt{\sum_{i=1}^{d}\sum_{j=m}^{l-1}\|d_{j,i}[N_1]\|_\tF^2}\right),\quad 3\leq l\leq L;\\
D_{l,i}:=&\begin{cases}W_l(\sigma^{'}(a_{l-1}^{(1)})*d_{l-1,i}),& 2\leq l\leq L,\\W_1(:,i)\Bone^\top,& l=1,\end{cases}\quad 1\leq i\leq d;\\
\eta_l:=&\begin{cases} 0,& l = L,\\
1, & 1\leq l\leq L-1.\end{cases}
\end{align*}

\bibliography{sn-bibliography.bib}


\begin{thebibliography}{28}
\ifx \bisbn   \undefined \def \bisbn  #1{ISBN #1}\fi
\ifx \binits  \undefined \def \binits#1{#1}\fi
\ifx \bauthor  \undefined \def \bauthor#1{#1}\fi
\ifx \batitle  \undefined \def \batitle#1{#1}\fi
\ifx \bjtitle  \undefined \def \bjtitle#1{#1}\fi
\ifx \bvolume  \undefined \def \bvolume#1{\textbf{#1}}\fi
\ifx \byear  \undefined \def \byear#1{#1}\fi
\ifx \bissue  \undefined \def \bissue#1{#1}\fi
\ifx \bfpage  \undefined \def \bfpage#1{#1}\fi
\ifx \blpage  \undefined \def \blpage #1{#1}\fi
\ifx \burl  \undefined \def \burl#1{\textsf{#1}}\fi
\ifx \doiurl  \undefined \def \doiurl#1{\url{https://doi.org/#1}}\fi
\ifx \betal  \undefined \def \betal{\textit{et al.}}\fi
\ifx \binstitute  \undefined \def \binstitute#1{#1}\fi
\ifx \binstitutionaled  \undefined \def \binstitutionaled#1{#1}\fi
\ifx \bctitle  \undefined \def \bctitle#1{#1}\fi
\ifx \beditor  \undefined \def \beditor#1{#1}\fi
\ifx \bpublisher  \undefined \def \bpublisher#1{#1}\fi
\ifx \bbtitle  \undefined \def \bbtitle#1{#1}\fi
\ifx \bedition  \undefined \def \bedition#1{#1}\fi
\ifx \bseriesno  \undefined \def \bseriesno#1{#1}\fi
\ifx \blocation  \undefined \def \blocation#1{#1}\fi
\ifx \bsertitle  \undefined \def \bsertitle#1{#1}\fi
\ifx \bsnm \undefined \def \bsnm#1{#1}\fi
\ifx \bsuffix \undefined \def \bsuffix#1{#1}\fi
\ifx \bparticle \undefined \def \bparticle#1{#1}\fi
\ifx \barticle \undefined \def \barticle#1{#1}\fi
\bibcommenthead
\ifx \bconfdate \undefined \def \bconfdate #1{#1}\fi
\ifx \botherref \undefined \def \botherref #1{#1}\fi
\ifx \url \undefined \def \url#1{\textsf{#1}}\fi
\ifx \bchapter \undefined \def \bchapter#1{#1}\fi
\ifx \bbook \undefined \def \bbook#1{#1}\fi
\ifx \bcomment \undefined \def \bcomment#1{#1}\fi
\ifx \oauthor \undefined \def \oauthor#1{#1}\fi
\ifx \citeauthoryear \undefined \def \citeauthoryear#1{#1}\fi
\ifx \endbibitem  \undefined \def \endbibitem {}\fi
\ifx \bconflocation  \undefined \def \bconflocation#1{#1}\fi
\ifx \arxivurl  \undefined \def \arxivurl#1{\textsf{#1}}\fi
\csname PreBibitemsHook\endcsname

\bibitem[\protect\citeauthoryear{Du et~al.}{2019}]{Du2019}
\begin{bchapter}
\bauthor{\bsnm{Du}, \binits{S.S.}},
\bauthor{\bsnm{Zhai}, \binits{X.}},
\bauthor{\bsnm{Poczos}, \binits{B.}},
\bauthor{\bsnm{Singh}, \binits{A.}}:
\bctitle{Gradient descent provably optimizes over-parameterized neural networks}.
In: \bbtitle{International Conference on Learning Representations}
(\byear{2019})
\end{bchapter}
\endbibitem

\bibitem[\protect\citeauthoryear{Allen-Zhu et~al.}{2019}]{Allen-Zhu2019}
\begin{bchapter}
\bauthor{\bsnm{Allen-Zhu}, \binits{Z.}},
\bauthor{\bsnm{Li}, \binits{Y.}},
\bauthor{\bsnm{Song}, \binits{Z.}}:
\bctitle{A convergence theory for deep learning via over-parameterization}.
In: \bbtitle{Proceedings of the 36th International Conference on Machine Learning},
vol. \bseriesno{97},
pp. \bfpage{242}--\blpage{252}
(\byear{2019})
\end{bchapter}
\endbibitem

\bibitem[\protect\citeauthoryear{Zou and Gu}{2019}]{Zou2019}
\begin{bchapter}
\bauthor{\bsnm{Zou}, \binits{D.}},
\bauthor{\bsnm{Gu}, \binits{Q.}}:
\bctitle{An improved analysis of training over-parameterized deep neural networks}.
In: \bbtitle{Advances in Neural Information Processing Systems}
(\byear{2019})
\end{bchapter}
\endbibitem

\bibitem[\protect\citeauthoryear{Du et~al.}{2019}]{Du2019_2}
\begin{bchapter}
\bauthor{\bsnm{Du}, \binits{S.S.}},
\bauthor{\bsnm{Lee}, \binits{J.}},
\bauthor{\bsnm{Li}, \binits{H.}},
\bauthor{\bsnm{Wang}, \binits{L.}},
\bauthor{\bsnm{Zhai}, \binits{X.}}:
\bctitle{Gradient descent finds global minima of deep neural networks}.
In: \bbtitle{Proceedings of the 36th International Conference on Machine Learning},
vol. \bseriesno{97},
pp. \bfpage{1675}--\blpage{1685}
(\byear{2019})
\end{bchapter}
\endbibitem

\bibitem[\protect\citeauthoryear{Allen-Zhu and Y.~Li}{2019}]{Allen-Zhu2019_2}
\begin{bchapter}
\bauthor{\bsnm{Allen-Zhu}, \binits{Z.}},
\bauthor{\bsnm{Y.~Li}, \binits{Y.L.}}:
\bctitle{Learning and generalization in overparameterized neural networks, going beyond two layers}.
In: \bbtitle{Advances in Neural Information Processing Systems}
(\byear{2019})
\end{bchapter}
\endbibitem

\bibitem[\protect\citeauthoryear{E et~al.}{2019}]{E2019}
\begin{barticle}
\bauthor{\bsnm{E}, \binits{W.}},
\bauthor{\bsnm{Ma}, \binits{C.}},
\bauthor{\bsnm{Wu}, \binits{L.}}:
\batitle{A comparative analysis of optimization and generalization properties of two-layer neural network and random feature models under gradient descent dynamics}.
\bjtitle{Sci. China Math}
\bvolume{63}(\bissue{7}),
\bfpage{1235}--\blpage{1258}
(\byear{2019})
\end{barticle}
\endbibitem

\bibitem[\protect\citeauthoryear{Zhou et~al.}{2021}]{Zhou2021}
\begin{bchapter}
\bauthor{\bsnm{Zhou}, \binits{M.}},
\bauthor{\bsnm{Ge}, \binits{R.}},
\bauthor{\bsnm{Jin}, \binits{C.}}:
\bctitle{A local convergence theory for mildly over-parameterized two-layer neural network}.
In: \bbtitle{Proceedings of Thirty Fourth Conference on Learning Theory},
vol. \bseriesno{139},
pp. \bfpage{4577}--\blpage{4632}
(\byear{2021})
\end{bchapter}
\endbibitem

\bibitem[\protect\citeauthoryear{Oymak and Soltanolkotabi}{2020}]{Oymak2020}
\begin{barticle}
\bauthor{\bsnm{Oymak}, \binits{S.}},
\bauthor{\bsnm{Soltanolkotabi}, \binits{M.}}:
\batitle{Toward moderate overparameterization: Global convergence guarantees for training shallow neural networks}.
\bjtitle{IEEE J. Sel. Areas Inf. Theory}
\bvolume{1}(\bissue{1}),
\bfpage{84}--\blpage{105}
(\byear{2020})
\end{barticle}
\endbibitem

\bibitem[\protect\citeauthoryear{Raissi et~al.}{2019}]{Raissi2019}
\begin{barticle}
\bauthor{\bsnm{Raissi}, \binits{M.}},
\bauthor{\bsnm{Perdikaris}, \binits{P.}},
\bauthor{\bsnm{Karniadakis}, \binits{G.E.}}:
\batitle{Physics-informed neural networks: {A} deep learning framework for solving forward and inverse problems involving nonlinear partial differential equations}.
\bjtitle{J. Comput. Phys.}
\bvolume{378},
\bfpage{686}--\blpage{707}
(\byear{2019})
\end{barticle}
\endbibitem

\bibitem[\protect\citeauthoryear{Rao et~al.}{2020}]{Rao2020}
\begin{barticle}
\bauthor{\bsnm{Rao}, \binits{C.}},
\bauthor{\bsnm{Sun}, \binits{H.}},
\bauthor{\bsnm{Liu}, \binits{Y.}}:
\batitle{Physics-informed deep learning for incompressible laminar flows}.
\bjtitle{Theor. Appl. Mech. Lett.}
\bvolume{10}(\bissue{3}),
\bfpage{207}--\blpage{212}
(\byear{2020})
\end{barticle}
\endbibitem

\bibitem[\protect\citeauthoryear{Cai et~al.}{2021}]{Cai2021}
\begin{barticle}
\bauthor{\bsnm{Cai}, \binits{S.}},
\bauthor{\bsnm{Wang}, \binits{Z.}},
\bauthor{\bsnm{Wang}, \binits{S.}},
\bauthor{\bsnm{Perdikaris}, \binits{P.}},
\bauthor{\bsnm{Karniadakis}, \binits{G.E.}}:
\batitle{Physics-informed neural networks for heat transfer problems}.
\bjtitle{J. Heat Transf.}
\bvolume{143}(\bissue{6}),
\bfpage{060801}
(\byear{2021})
\end{barticle}
\endbibitem

\bibitem[\protect\citeauthoryear{Pang et~al.}{2019}]{Pang2019}
\begin{barticle}
\bauthor{\bsnm{Pang}, \binits{G.}},
\bauthor{\bsnm{Lu}, \binits{L.}},
\bauthor{\bsnm{Karniadakis}, \binits{G.E.}}:
\batitle{{fPINNs}: {Fractional} physics-informed neural networks}.
\bjtitle{SIAM J. Sci. Comput.}
\bvolume{41}(\bissue{4}),
\bfpage{2603}--\blpage{2626}
(\byear{2019})
\end{barticle}
\endbibitem

\bibitem[\protect\citeauthoryear{Jagtap and Karniadakis}{2021}]{Jagtap2021}
\begin{bchapter}
\bauthor{\bsnm{Jagtap}, \binits{A.D.}},
\bauthor{\bsnm{Karniadakis}, \binits{G.E.}}:
\bctitle{Extended physics-informed neural networks {(XPINNs)}: {A} generalized space-time domain decomposition based deep learning framework for nonlinear partial differential equations}.
In: \bbtitle{AAAI Spring Symposium: MLPS},
vol. \bseriesno{10}
(\byear{2021})
\end{bchapter}
\endbibitem

\bibitem[\protect\citeauthoryear{Chiu et~al.}{2022}]{Chiu2022}
\begin{barticle}
\bauthor{\bsnm{Chiu}, \binits{P.-H.}},
\bauthor{\bsnm{Wong}, \binits{J.C.}},
\bauthor{\bsnm{Ooi.}, \binits{C.}},
\bauthor{\bsnm{Dao}, \binits{M.H.}},
\bauthor{\bsnm{Ong}, \binits{Y.-S.}}:
\batitle{{CAN-PINN: A} fast physics-informed neural network based on coupled-automatic–numerical differentiation method}.
\bjtitle{Comput. Methods Appl. Mech. Eng.}
\bvolume{395},
\bfpage{114909}
(\byear{2022})
\end{barticle}
\endbibitem

\bibitem[\protect\citeauthoryear{Gao et~al.}{2023}]{Gao2023}
\begin{bchapter}
\bauthor{\bsnm{Gao}, \binits{Y.}},
\bauthor{\bsnm{Gu}, \binits{Y.}},
\bauthor{\bsnm{Ng}, \binits{M.}}:
\bctitle{Gradient descent finds the global optima of two-layer physics-informed neural networks}.
In: \bbtitle{Proceedings of the 40th International Conference on Machine Learning},
vol. \bseriesno{202},
pp. \bfpage{10676}--\blpage{10707}
(\byear{2023})
\end{bchapter}
\endbibitem

\bibitem[\protect\citeauthoryear{Luo and Yang}{2020}]{Luo2020}
\begin{botherref}
\oauthor{\bsnm{Luo}, \binits{T.}},
\oauthor{\bsnm{Yang}, \binits{H.}}:
Two-layer neural networks for partial differential equations: Optimization and generalization theory.
https://arxiv.org/abs/2006.15733
(2020)
\end{botherref}
\endbibitem

\bibitem[\protect\citeauthoryear{R{\"o}gnvaldsson}{1994}]{Rognvaldsson1994}
\begin{barticle}
\bauthor{\bsnm{R{\"o}gnvaldsson}, \binits{T.}}:
\batitle{On langevin updating in multilayer perceptrons}.
\bjtitle{Neural Comput.}
\bvolume{6}(\bissue{5}),
\bfpage{916}--\blpage{926}
(\byear{1994})
\end{barticle}
\endbibitem

\bibitem[\protect\citeauthoryear{Erhan et~al.}{2009}]{Erhan2009}
\begin{bchapter}
\bauthor{\bsnm{Erhan}, \binits{D.}},
\bauthor{\bsnm{Manzagol}, \binits{P.A.}},
\bauthor{\bsnm{Bengio}, \binits{Y.}},
\bauthor{\bsnm{Bengio}, \binits{S.}},
\bauthor{\bsnm{Vincent}, \binits{P.}}:
\bctitle{The difficulty of training deep architectures and the effect of unsupervised pre-training}.
In: \bbtitle{Proceedings of the 12th Int. Workshop on Artificial Intelligence and Statistics},
vol. \bseriesno{5},
pp. \bfpage{153}--\blpage{160}
(\byear{2009})
\end{bchapter}
\endbibitem

\bibitem[\protect\citeauthoryear{Carreira-Perpinan and Wang}{2014}]{Carreira-Perpinan2014}
\begin{bchapter}
\bauthor{\bsnm{Carreira-Perpinan}, \binits{M.}},
\bauthor{\bsnm{Wang}, \binits{W.}}:
\bctitle{Distributed optimization of deeply nested systems}.
In: \bbtitle{Proceedings of the Seventeenth International Conference on Artificial Intelligence and Statistics},
vol. \bseriesno{33},
pp. \bfpage{10}--\blpage{19}
(\byear{2014})
\end{bchapter}
\endbibitem

\bibitem[\protect\citeauthoryear{Taylor et~al.}{2016}]{Taylor2016}
\begin{bchapter}
\bauthor{\bsnm{Taylor}, \binits{G.}},
\bauthor{\bsnm{Burmeister}, \binits{R.}},
\bauthor{\bsnm{Xu}, \binits{Z.}},
\bauthor{\bsnm{Singh}, \binits{B.}},
\bauthor{\bsnm{Patel}, \binits{A.}},
\bauthor{\bsnm{Goldstein}, \binits{T.}}:
\bctitle{Training neural networks without gradients: A scalable admm approach}.
In: \bbtitle{Proceedings of The 33rd International Conference on Machine Learning},
vol. \bseriesno{48},
pp. \bfpage{2722}--\blpage{2731}
(\byear{2016})
\end{bchapter}
\endbibitem

\bibitem[\protect\citeauthoryear{Wang et~al.}{2019}]{Wang2019}
\begin{bchapter}
\bauthor{\bsnm{Wang}, \binits{J.}},
\bauthor{\bsnm{Yu}, \binits{F.}},
\bauthor{\bsnm{Chen}, \binits{X.}},
\bauthor{\bsnm{Zhao}, \binits{L.}}:
\bctitle{Admm for efficient deep learning with global convergence}.
In: \bbtitle{Proceedings of the 25th ACM SIGKDD International Conference on Knowledge Discovery \& Data Mining},
pp. \bfpage{111}--\blpage{119}
(\byear{2019})
\end{bchapter}
\endbibitem

\bibitem[\protect\citeauthoryear{Gao et~al.}{2019}]{Gao2019}
\begin{barticle}
\bauthor{\bsnm{Gao}, \binits{Y.}},
\bauthor{\bsnm{Zhao}, \binits{L.}},
\bauthor{\bsnm{Wu}, \binits{L.}},
\bauthor{\bsnm{Ye}, \binits{Y.}},
\bauthor{\bsnm{Xiong}, \binits{H.}},
\bauthor{\bsnm{Yang}, \binits{C.}}:
\batitle{Incomplete label multi-task deep learning for spatio-temporal event subtype forecasting}.
\bjtitle{Proceedings of the AAAI conference on artificial intelligence}
\bvolume{33}(\bissue{01}),
\bfpage{3638}--\blpage{3646}
(\byear{2019})
\end{barticle}
\endbibitem

\bibitem[\protect\citeauthoryear{Yang et~al.}{2018}]{Yang2018}
\begin{barticle}
\bauthor{\bsnm{Yang}, \binits{Y.}},
\bauthor{\bsnm{Sun}, \binits{J.}},
\bauthor{\bsnm{Li}, \binits{H.}},
\bauthor{\bsnm{Xu}, \binits{Z.}}:
\batitle{{ADMM-CSNet}: A deep learning approach for image compressive sensing}.
\bjtitle{IEEE Trans. Pattern Anal. Mach. Intell.}
\bvolume{42}(\bissue{3}),
\bfpage{521}--\blpage{538}
(\byear{2018})
\end{barticle}
\endbibitem

\bibitem[\protect\citeauthoryear{Sun et~al.}{2016}]{Sun2016}
\begin{bchapter}
\bauthor{\bsnm{Sun}, \binits{J.}},
\bauthor{\bsnm{Li}, \binits{H.}},
\bauthor{\bsnm{Xu}, \binits{Z.}}:
\bctitle{Deep {ADMM-Net} for compressive sensing {MRI}}.
In: \bbtitle{Advances in Neural Information Processing Systems}
(\byear{2016})
\end{bchapter}
\endbibitem

\bibitem[\protect\citeauthoryear{Song et~al.}{2023}]{Song2023}
\begin{botherref}
\oauthor{\bsnm{Song}, \binits{Y.}},
\oauthor{\bsnm{Yuan}, \binits{X.}},
\oauthor{\bsnm{Yue}, \binits{H.}}:
The {ADMM-PINNs} Algorithmic Framework for Nonsmooth {PDE}-Constrained Optimization: {A} Deep Learning Approach.
https://arxiv.org/abs/2302.08309
(2023)
\end{botherref}
\endbibitem

\bibitem[\protect\citeauthoryear{Neyshabur et~al.}{2017}]{Neyshabur2017}
\begin{botherref}
\oauthor{\bsnm{Neyshabur}, \binits{B.}},
\oauthor{\bsnm{Tomioka}, \binits{R.}},
\oauthor{\bsnm{Salakhutdinov}, \binits{R.}},
\oauthor{\bsnm{Srebro}, \binits{N.}}:
Geometry of optimization and implicit regularization in deep learning.
http://arxiv.org/abs/1705.03071
(2017)
\end{botherref}
\endbibitem

\bibitem[\protect\citeauthoryear{Lei et~al.}{2018}]{Lei2018}
\begin{botherref}
\oauthor{\bsnm{Lei}, \binits{D.}},
\oauthor{\bsnm{Sun}, \binits{Z.}},
\oauthor{\bsnm{Xiao}, \binits{Y.}},
\oauthor{\bsnm{Wang}, \binits{W.Y.}}:
Implicit regularization of stochastic gradient descent in natural language processing: {Observations} and implications.
http://arxiv.org/abs/1811.00659
(2018)
\end{botherref}
\endbibitem

\bibitem[\protect\citeauthoryear{Cao et~al.}{2021}]{Cao2021}
\begin{bchapter}
\bauthor{\bsnm{Cao}, \binits{Y.}},
\bauthor{\bsnm{Fang}, \binits{Z.}},
\bauthor{\bsnm{Wu}, \binits{Y.}},
\bauthor{\bsnm{Zhou}, \binits{D.-X.}},
\bauthor{\bsnm{Gu}, \binits{Q.}}:
\bctitle{Towards understanding the spectral bias of deep learning}.
In: \bbtitle{Proceedings of the Thirtieth International Joint Conference on Artificial Intelligence},
pp. \bfpage{2205}--\blpage{2211}
(\byear{2021})
\end{bchapter}
\endbibitem

\end{thebibliography}

\end{document}